\documentclass{article} %
\usepackage{iclr2027_conference,times}

\usepackage{amsmath,amsfonts,bm}

\def\eqref#1{equation~\ref{#1}}

\def\1{\bm{1}}

\DeclareMathAlphabet{\mathsfit}{\encodingdefault}{\sfdefault}{m}{sl}
\SetMathAlphabet{\mathsfit}{bold}{\encodingdefault}{\sfdefault}{bx}{n}

\def\sR{{\mathbb{R}}}
\def\sS{{\mathbb{S}}}
\def\sT{{\mathbb{T}}}

\usepackage{hyperref}
\usepackage{url}

\usepackage{amsthm}
\newtheorem{proposition}{Proposition}

\usepackage{booktabs}
\usepackage{graphicx}
\usepackage{algorithm}
\usepackage{algpseudocode}
\usepackage{array}
\usepackage{multirow}
\usepackage{booktabs}
\usepackage{makecell}

\title{Light Entropic Optimal Transport\\on Riemannian Manifolds}

\author{Xavier Aramayo-Carrasco\\
    Applied AI Institute, \\
    Moscow, Russia
    \And
    Petr Mokrov \\
    Applied AI Institute,\\
    Moscow, Russia
    \And
    Alexander Korotin \\
    Applied AI Institute\\
    AXXX,\\
    Moscow, Russia
}

\iclrfinalcopy %
\begin{document}

\vspace{-5mm}
\maketitle

\vspace{-5mm}

\begin{abstract}
Entropic Optimal Transport (EOT) has become a practical framework for learning stochastic couplings between complex distributions, with applications in generative modeling and domain adaptation. However, most EOT solvers are designed for Euclidean spaces, while manifold extensions remain limited and often rely on costly iterative methods, simulated dynamics, or generic neural models that do not fully exploit the underlying geometry. We introduce ManifoldLightOT, a light approach for learning kernel-induced EOT couplings directly on common manifolds. Using the kernel form of the EOT solution, we construct geometry-specific Gibbs kernels together with compatible potential parameterizations for spheres, tori, $\mathrm{SO}(3)$, and $\mathrm{SE}(3)$. These choices yield closed-form normalization and directly sampleable conditional distributions. Our formulation naturally extends to products of manifolds, making it applicable to more complex geometries. The parameters of the potentials are optimized directly from samples using Monte Carlo estimates of the learning objective. Through synthetic and real-world experiments, we show that ManifoldLightOT often outperforms existing manifold OT methods while retaining direct sampling.
\end{abstract}

\vspace{-2mm}

\section{Introduction}
\label{section:introduction}

\vspace{-1mm}

In recent years, developing computationally efficient solvers for Optimal Transport (OT) and Entropic Optimal Transport (EOT) problems \citep{cuturi2013sinkhorn, villani2008optimal} has attracted considerable interest. OT seeks a coupling between prescribed source and target distributions that minimizes an expected transportation cost, while EOT augments this objective with an entropic regularization term. A well-established class of solvers operates on \textbf{discrete} probability distributions, representing the coupling as a matrix whose dimensions depend on the numbers of source and target samples; see \citet{peyre2019computational} for a survey. In this work, we focus on \textbf{continuous} solvers \citep{gushchin2023buildingbridgeschrodingercontinuous, korotin2021neuraloptimaltransportsolvers}, which learn conditional distributions or transport maps (plans), typically parameterized by neural networks, from samples of the source and target distributions. The EOT formulation is equivalent to the static Schr\"odinger bridge (SB) problem \citep{schrodinger1931umkehrung, schrodinger1932theorie, leonard2013survey} when the transport cost is induced by the reference measure of the SB.

Several neural approaches have been proposed in recent years to solve the EOT or SB problems \citep{mokrov2024energyguided, debortoli2023diffusionschrodingerbridgeapplications, tong2024simulationfreeschrodingerbridgesscore, gushchin2023entropic, korotin2024light, persiianov2026learningpopulationdynamicsinverse}. These solvers support out-of-sample estimation, enabling applications such as image generation, image-to-image translation, domain adaptation, and biological tasks including modeling cell dynamics and predicting cellular responses to perturbations. However, these methods are formulated in Euclidean spaces, where source and target samples are elements of $\sR^D$. However, these methods may not be directly adjusted when data are supported on curved manifolds. %

Switching on manifold support has proved effective across several generative modeling frameworks, including flow matching \citep{chen2024flowmatchinggeneralgeometries}, continuous normalizing flows \citep{mathieu2020riemanniancontinuousnormalizingflows}, and score-based diffusion models \citep{debortoli2022riemannianscorebasedgenerativemodelling}. Manifold-aware formulations have also been developed for neural OT and EOT, notably Riemannian Neural OT (RNOT) \citep{micheli2026riemannianneuraloptimaltransport} and Entropic Riemannian Neural OT (ERNOT) \citep{micheli2026entropicriemannianneuraloptimal}, respectively. In practice, RNOT approximates the \(c\)-transform through an additional inner optimization, while ERNOT relies on minibatch or finite-support approximations of the soft \(c\)-transform and the induced Gibbs conditionals. These approximations can introduce additional optimization error, sampling bias, and sensitivity to the chosen target support in high-dimensional settings. Motivated by these observations, we propose ManifoldLightOT, an EOT solver with geometry-specific potential parameterizations for several classes of Riemannian manifolds. Following the terminology of LightSB \citep{korotin2024light}, we use the word ``light'' to denote parameterizations that yield closed-form expressions for the normalizing term and conditional distribution of the EOT coupling.

Our \textbf{main contribution} is ManifoldLightOT, a light solver for kernel-induced EOT problems on spheres, tori, \(\mathrm{SO}(3)\), and \(\mathrm{SE}(3)\). The method combines geometry-specific Gibbs kernels with compatible potential families, yielding closed-form normalization functions and directly sampleable conditional distributions. We further show that these tractability properties are preserved under Cartesian products, extending the construction to heterogeneous product manifolds. The resulting solver is trained using Monte Carlo estimates of the learning objective from unpaired source and target samples, without additional inner optimization problems.

We evaluate ManifoldLightOT on synthetic transport problems across the considered manifolds and on real-world manifold-valued data, comparing its performance with other existing OT methods.

\textbf{Notation.} Let \(\mathcal M\) denote the underlying manifold, \(m\) its Riemannian volume measure, which reduces to the standard Lebesgue measure when \(\mathcal M=\mathbb R^d\), and \(\mathcal P(\mathcal M)\) the set of probability measures on \(\mathcal M\). We denote the source and target measures by \(\mu,\nu\in\mathcal P(\mathcal M)\), assumed to be absolutely continuous with respect to \(m\), and their corresponding points by \(x,y\in\mathcal M\). We use \(\mu(x)\) and \(\nu(y)\) to denote their densities with respect to \(m\). For a coupling \(q\in\Pi(\mu,\nu)\), we denote its density with respect to \(m\otimes m\) by \(q(x,y)\), and its conditional density by \(q(y|x)\). Accordingly, \(q(x,y)=\mu(x)q(y|x)\), the marginal constraints require \(\int_{\mathcal M}q(x,y)\,dm(y)=\mu(x)\) and \(\int_{\mathcal M}q(x,y)\,dm(x)=\nu(y)\).

\vspace{-2mm}

\section{Background}
\label{section:background}

\vspace{-2mm}

In this section, we recall the main concepts underlying our approach. We first review the formulation of entropic optimal transport on manifolds and the kernel characterization of its optimal coupling. We then recall its connection to the static Schrödinger bridge problem. Finally, we review the LightSB solver \citep{korotin2024light} as a representative light approach for parameterizing the optimal coupling, while noting that its construction is restricted to Euclidean spaces.

\textbf{Entropic Optimal Transport on Riemannian Manifolds.}
We consider the Entropic Optimal Transport (EOT) problem directly on a Riemannian manifold \(\mathcal M\). For simplicity, throughout the paper we refer to this manifold formulation simply as EOT. Let \(\mu,\nu\in\mathcal P(\mathcal M)\) be the source and target probability measures, respectively, and let \(\Pi(\mu,\nu)\) denote the set of couplings on \(\mathcal M\times\mathcal M\) with marginals \(\mu\) and \(\nu\). Given a transport cost \(c:\mathcal M\times\mathcal M\to\mathbb R\) and a regularization parameter \(\varepsilon>0\), the EOT problem \citep{cuturi2013sinkhorn,genevay2019entropy} is defined as

\vspace{-2mm}

\begin{equation}
    q^*
    \in
    \arg\min_{q\in\Pi(\mu,\nu)}
    \left\{
        \mathbb E_{(x,y)\sim q}
        \left[c(x,y)\right]
        -\varepsilon\mathrm{H}(q)
    \right\}.
    \label{equation:entropic_ot}
\end{equation}

\vspace{-1mm}

Here, \(\mathrm{H}(q)\) denotes the entropy of the coupling density \(q(x,y)\) with respect to the product Riemannian volume measure \(m\otimes m\). We assume the usual well-posedness conditions for continuous EOT, including absolute continuity, finite entropy, and finite expected cost. For \(\mathrm{SE}(3)\), we also assume finite translational second moments. While many manifold OT methods use the squared geodesic cost \(c(x,y)=\frac{1}{2}d_{\mathcal M}^2(x,y)\), our formulation uses geometry-specific kernel-induced costs.

\textbf{Kernel characterization of the EOT solution.}
The solution \(q^*\) of the EOT problem in ~\eqref{equation:entropic_ot} admits a Gibbs kernel representation \citep{genevay2019entropy,peyre2019computational}. In particular, for a given cost \(c\) and regularization parameter \(\varepsilon\), the associated Gibbs kernel is defined as $K_\varepsilon(x,y)\!=\!\exp\left(-c(x,y)/\varepsilon\right).$ With this kernel, the density of the optimal EOT coupling with respect to the product Riemannian volume measure \(m\otimes m\) admits the following form

\vspace{-2mm}

\begin{equation}
    q^*(x,y)
    =
    u^*(x)
    K_\varepsilon(x,y)
    v^*(y),
    \label{equation:eot_kernel_solution}
\end{equation}

\vspace{-1mm}

where \(u^*\) and \(v^*\) are positive scaling functions chosen such that \(q^*\) satisfies the marginal constraints \(\mu\) and \(\nu\). Using the source marginal constraint, the optimal coupling in ~\eqref{equation:eot_kernel_solution} can equivalently be expressed through its conditional distribution as

\vspace{-2mm}

\begin{equation}
    q^*(y|x)
    =
    \frac{
        1
    }{
        r^*(x)
    }K_\varepsilon(x,y)v^*(y),
    \label{equation:eot_conditional}
\end{equation}

where \(r^*(x)=\int_{\mathcal M}K_\varepsilon(x,y)v^*(y)\,dm(y)\) is the corresponding normalization function. The EOT problem in \eqref{equation:entropic_ot} is equivalent to the static Schr\"odinger bridge problem \citep{schrodinger1931umkehrung,schrodinger1932theorie,leonard2013survey} when the transport cost is induced by the reference measure from the SB.

\textbf{Light Parameterizations of the EOT Solution.}
Under the equivalence described above, EOT and static SB share the same optimal coupling \(q^*\). Following Light Schr\"odinger Bridge (LightSB) \citep{korotin2024light}, a parametric coupling \(q_\theta \approx q^*\) is learned by minimizing the Kullback-Leibler (KL) divergence \(\mathrm{KL}(q^*\,\|\,q_\theta)\), which is equivalent up to constants to minimizing the surrogate loss \(\mathbb E_{x\sim \mu}[\log r_\theta(x)]-\mathbb E_{y\sim \nu}[\log v_\theta(y)]\). Therefore, learning the conditional EOT solution in \eqref{equation:eot_conditional} can be approached by choosing a suitable parametric scaling function \(v_\theta\). In the Euclidean setting, LightSB uses a light parameterization for \(v_\theta\) as a learnable weighted mixture of Gaussian densities,

\vspace{-4mm}

\begin{equation}
v_\theta(y)=\sum_{k=1}^{K}\alpha_k\,\mathcal N(y| \mu_k,\varepsilon S_k).\notag
\end{equation}

\vspace{-2mm}

Here, the optimizable parameters are \(\theta=\{\alpha_k,\mu_k,S_k\}_{k=1}^{K}\), where \(\alpha_k>0\) are mixture weights, \(\mu_k\in\mathbb R^d\) are component means, and \(S_k\in\mathbb R^{d\times d}\) are positive-definite covariance parameters. This parameterization leads to closed-form expressions for the normalizing function \(r_\theta(x)\) and the conditional coupling \(q_\theta(y|x)\). Together with the surrogate loss above, these closed forms enable training from independent samples of \(\mu\) and \(\nu\), as well as direct sampling from \(q_\theta(y|x)\). However, this construction is based on Gaussian distributions and Euclidean identities, so it is not directly applicable to manifold-supported distributions.

\subsection{Continuous EOT Learning Setup}
\label{subsection:continuous_eot_setup}

\vspace{-2mm}

Although the EOT problem in \eqref{equation:entropic_ot} is formulated in terms of the source and target measures \(\mu\) and \(\nu\), in practical applications these distributions are typically not available explicitly but only through empirical samples. Below we formalize the continuous EOT setup considered in our work.

\noindent\fbox{\parbox{\dimexpr\linewidth-2\fboxsep-2\fboxrule\relax}{
We assume that the learner is given empirical datasets \(\{x^{(i)}\}_{i=1}^{N}\) and \(\{y^{(j)}\}_{j=1}^{M}\), with \(x^{(i)},y^{(j)}\in\mathcal M\), consisting of i.i.d. samples from unknown source and target measures \(\mu,\nu\in\mathcal P(\mathcal M)\), respectively. The two datasets are assumed to be independent and unpaired. The task is to use these samples to learn an approximation of the EOT solution \(q^*\in\Pi(\mu,\nu)\) in \eqref{equation:entropic_ot}, and in particular its conditional distribution \(q^*(y|x)\).
}}

This setup corresponds to continuous EOT learning from empirical samples and should be distinguished from discrete EOT. Discrete EOT solvers, such as Sinkhorn~\citep{cuturi2013sinkhorn} and gradient-based methods~\citep{dvurechensky2018computational}, treat the samples as empirical measures, \(\widehat{\mu}_N=\frac{1}{N}\sum_{i=1}^{N}\delta_{x^{(i)}}\) and \(\widehat{\nu}_M=\frac{1}{M}\sum_{j=1}^{M}\delta_{y^{(j)}}\), and return an \(N\times M\) coupling matrix. Such a solution does not define transport for unseen source points. In contrast, we use the samples as training data to approximate the EOT coupling between measures \(\mu\) and \(\nu\), yielding an out-of-sample model from which \(y^{\mathrm{new}}\) can be generated according to an approximation of \(q^*(y|x^{\mathrm{new}})\) for \(x^{\mathrm{new}}\sim\mu\).

Several works extend empirical OT or EOT solutions to unseen samples through map estimation, barycentric projection, or smoothing~\citep{hutter2021minimax,pooladian2021entropic,deb2021rates,manole2024pluginestimationsmoothoptimal, makkuva2020optimal}. However, they are specialized to quadratic Euclidean costs and rely on structures such as Brenier maps, Euclidean barycenters, or linear smoothers, which do not directly extend to our general kernel-induced costs on manifolds. Although SB provides an equivalent formulation of EOT under a suitable reference process~\citep{leonard2013survey}, but existing continuous SB solvers are typically tied to Euclidean reference processes and kernels~\citep{debortoli2023diffusionschrodingerbridgeapplications,gushchin2023entropic,tong2024simulationfreeschrodingerbridgesscore,korotin2024light}. Adapting them to our geometry-specific \(K_\varepsilon\) would require additional manifold-specific modifications, so we do not include them as direct baselines. We review general-cost continuous EOT solvers in \S\ref{section:related_works}.

\section{Manifold Entropic Optimal Transport Solver}
\label{section:manifold_eot_solver}

The design of ManifoldLightOT is motivated by the observation that tractability (parameterized closed-form solution, direct sampling) in continuous EOT depends on the compatibility between the Gibbs kernel and its associated probability family. In the Euclidean setting, LightSB~\citep{korotin2024light} exploits this compatibility for the quadratic cost through Gaussian kernels and Gaussian mixtures. We adopt this insight as a design criterion for manifold-valued transport. Our method specifies the EOT problem through geometry-aware Gibbs kernels. Each kernel \(K_\varepsilon\) induces the corresponding cost \(c_\varepsilon(x,y)=-\varepsilon\log K_\varepsilon(x,y)\), up to additive constants. The goal is therefore not to approximate squared-geodesic EOT on every manifold, but to construct tractable kernel-induced EOT problems whose associated probability families admit closed-form normalization and conditional sampling. We apply this principle for spheres, tori, \(\mathrm{SO}(3)\), \(\mathrm{SE}(3)\), and their Cartesian products. We derive the learning objective in \S\ref{subsection:learning_objective}, introduce the manifold-specific constructions in \S\ref{subsection:tractable_parameterizations}, and describe training and sampling in \S\ref{subsection:training_sampling}.

\vspace{-1mm}

\subsection{Learning Objective}
\label{subsection:learning_objective}

We first describe how the parametric coupling \(q_\theta\) is induced. Given a positive scaling function \(v_\theta\), with finite-dimensional parameters \(\theta\), we define
\(r_\theta(x)=\int_{\mathcal M}K_\varepsilon(x,y)v_\theta(y)\,dm(y)\)
and
\(q_\theta(y|x)=K_\varepsilon(x,y)v_\theta(y)/r_\theta(x)\).
Together with the source marginal, this conditional defines the joint coupling \(q_\theta(x,y)=\mu(x)q_\theta(y|x)\). Thus, \(q_\theta\) is not parameterized as an unrestricted neural coupling; it is determined by the choice of \(v_\theta\) and the fixed Gibbs kernel \(K_\varepsilon\).

Following the same LightSB principle, we learn \(\theta\) by minimizing the forward KL divergence from the optimal EOT coupling \(q^*\) to the induced coupling \(q_\theta\):

\vspace{-2mm}

\begin{equation}
    \min_{\theta}
    \mathrm{KL}\left(q^*\,\|\,q_\theta\right),
    \label{equation:kl_learning_objective}
\end{equation}

\vspace{-2mm}

where optimal coupling $q^*$ has form \eqref{equation:eot_kernel_solution}. Although the optimal coupling \(q^*\) is unknown, the terms in \eqref{equation:kl_learning_objective} that depend on \(\theta\) can be expressed using only its source and target marginals. In particular, \citet[Proposition~3.1]{korotin2024light} show that \eqref{equation:kl_learning_objective} has the same minimizers as

\vspace{-2mm}
\begin{equation}
    \min_{\theta}\mathcal L(\theta),
    \qquad
    \mathcal L(\theta)
    \overset{\text{def}}{=}
    \mathbb E_{x\sim \mu}
    \left[\log r_\theta(x)\right]
    -
    \mathbb E_{y\sim \nu}
    \left[\log v_\theta(y)\right].
    \label{equation:lightsb_learning_objective}
\end{equation}

\vspace{-1mm}

Therefore, the objective can be estimated from independent, unpaired samples from \(\mu\) and \(\nu\), without requiring access to or samples from the unknown optimal coupling \(q^*\). Importantly, this objective remains valid on manifolds, since its derivation relies only on the marginal constraints of the optimal coupling and not on the particular geometry of the underlying state space. %

\vspace{-1mm}

\subsection{Tractable Manifold-aware Parameterizations}
\label{subsection:tractable_parameterizations}

In this subsection, we introduce geometry-specific parameterizations \(v_\theta\) of the optimal scaling function \(v^*\) that yield closed-form expressions for the corresponding normalization function \(r_\theta\) and conditional distribution \(q_\theta(y|x)\). We consider spheres (\(\sS^{d-1}\)), tori (\(\sT^d\)), the rotation group (\(\mathrm{SO}(3)\)), the special Euclidean group (\(\mathrm{SE}(3)\)), and Cartesian products of these manifolds. We first state a general construction underlying all the proposed parameterizations. 

\begin{proposition}[Tractable light parameterizations]
\label{proposition:parameterization}
Let \(K_\varepsilon\) be a positive kernel on \(\mathcal M\times\mathcal M\), and let \(v_{\theta,k}:\mathcal M\to\mathbb R_{+}\) be positive potential components such that, for each \(k\), the product \(K_\varepsilon(x,y)v_{\theta,k}(y)\) admits a closed-form normalization
\(r_{\theta,k}(x)=\int_{\mathcal M}K_\varepsilon(x,y)v_{\theta,k}(y)\,dm(y)\)
and the normalized conditional
\(q_{\theta,k}(y|x)=K_\varepsilon(x,y)v_{\theta,k}(y)/r_{\theta,k}(x)\)
can be sampled directly. For the mixture parameterization \(v_\theta(y)=\sum_{k=1}^{K}\alpha_k v_{\theta,k}(y)\), with \(\alpha_k>0\), it follows that

\vspace{-3mm}

\begin{equation}
    r_\theta(x)
    =
    \sum_{k=1}^{K}\alpha_k r_{\theta,k}(x),
    \qquad
    q_\theta(y|x)
    =
    \sum_{k=1}^{K}
    w_{\theta,k}(x)
    q_{\theta,k}(y|x),
\end{equation}

\vspace{-2mm}

where \(w_{\theta,k}(x)=\alpha_k r_{\theta,k}(x)/r_\theta(x)\). Therefore, the compatibility of \(K_\varepsilon\) and \(v_{\theta,k}\) yields closed-form normalization and direct conditional sampling for \(q_\theta\).
\end{proposition}

\vspace{-2mm}

It is important to note that the construction in Proposition~\ref{proposition:parameterization} extends the compatibility principle behind the Gaussian-mixture adjusted-potential parameterization used in LightSB. In the Euclidean quadratic-cost setting, this tractable construction can be written using the adjusted inner-product Gibbs kernel with Gaussian-mixture adjusted potentials, i.e., \(\exp(-\|x-y\|^2/(2\varepsilon))=\exp(-\|x\|^2/(2\varepsilon))\exp(\langle x,y\rangle/\varepsilon)\exp(-\|y\|^2/(2\varepsilon))\), so the usual Gaussian Gibbs kernel differs from the inner-product kernel only by separate factors depending on \(x\) and \(y\). We next apply this idea to several common non-Euclidean manifolds by constructing compatible choices of \(K_\varepsilon\), \(v_{\theta,k}\), and induced costs \(c(x,y)\) for which the normalization and conditional sampling conditions of Proposition~\ref{proposition:parameterization} are satisfied in closed form. All proofs appear in Appendix~\ref{appendix:proofs}.

\vspace{-2mm}
\paragraph{Sphere.}
Let \(\mathcal M=\sS^{d-1}\), with \(x,y\in\sS^{d-1}\). For \(g\in\mathbb R^d\), define the von Mises-Fisher normalizing function \(Z_d(g):=\int_{\sS^{d-1}}\exp(g^\top x)\,dm(x)\), and let \(\operatorname{vMF}(g)\) denote the corresponding von Mises-Fisher distribution with natural parameter \(g\).

\vspace{-3mm}
\begin{equation}
    K_\varepsilon(x,y)=\exp\!\left(\frac{x^\top y}{\varepsilon}\right),
    \qquad
    v_{\theta,k}(y)=\exp(f_k^\top y).
    \label{equation:sphere_kernel_potential}
\end{equation}

\vspace{-3mm}

\begin{equation}
    r_{\theta,k}(x)=Z_d\!\left(f_k+\frac{x}{\varepsilon}\right),
    \qquad
    q_{\theta,k}(\cdot|x)=\operatorname{vMF}\!\left(f_k+\frac{x}{\varepsilon}\right).
    \label{equation:sphere_normalizer_conditional}
\end{equation}

\vspace{-2mm}

The product \(K_\varepsilon(x,y)v_{\theta,k}(y)\) remains in the von Mises-Fisher family, yielding the closed-form normalization. The learnable parameters are \(\theta=\{\alpha_k,f_k\}_{k=1}^{K}\), with \(\alpha_k>0\) and \(f_k\in\mathbb R^d\). From \(K_\varepsilon(x,y)=\exp[-c(x,y)/\varepsilon]\), this kernel induces the cost \(c(x,y)=-x^\top y\). Conditional samples are drawn directly from \(\operatorname{vMF}(f_k+x/\varepsilon)\).

\vspace{-2mm}
\paragraph{Torus.}
Let \(\mathcal M=\sT^d\), with \(x,y\in[0,2\pi)^d\) denoting their angular coordinates and \(x^j,y^j\) their \(j\)-th coordinates. Define \(h_{kj}(x):=(a_{kj}+\varepsilon^{-1}\cos x^j,\,b_{kj}+\varepsilon^{-1}\sin x^j)^\top\) and \(\rho_{kj}(x):=\|h_{kj}(x)\|_2\). Let \(I_0\) be the modified Bessel function of the first kind and \(\operatorname{vM}(h)\) the circular von Mises distribution with natural parameter \(h\in\mathbb R^2\).

\vspace{-3mm}

\begin{equation}
    K_\varepsilon(x,y)=\prod_{j=1}^{d}\exp\!\left[\frac{\cos(x^j-y^j)}{\varepsilon}\right],
    \quad
    v_{\theta,k}(y)=\prod_{j=1}^{d}\exp\!\left[a_{kj}\cos y^j+b_{kj}\sin y^j\right].
    \label{equation:torus_kernel_potential}
\end{equation}

\vspace{-3mm}

\begin{equation}
    r_{\theta,k}(x)=\prod_{j=1}^{d}2\pi I_0\!\left(\rho_{kj}(x)\right),
    \qquad
    q_{\theta,k}(\cdot|x)=\prod_{j=1}^{d}\operatorname{vM}\!\left(h_{kj}(x)\right).
    \label{equation:torus_normalizer_conditional}
\end{equation}

\vspace{-2mm}

The learnable parameters are \(\theta=\{\alpha_k,a_k,b_k\}_{k=1}^{K}\), with \(\alpha_k\!>\!0\) and \(a_k,b_k\in\mathbb R^d\). The induced cost is {\(c(x,y)=-\sum_{j=1}^{d}\cos(x^j-y^j)\)}, which is equivalent to \(c(x,y)=\sum_{j=1}^{d}[1-\cos(x^j-y^j)]\) up to an additive constant. Sampling is done independently for each coordinate from \(\operatorname{vM}(h_{kj}(x))\).

\vspace{-2mm}

\paragraph{Rotation group.}
Let \(\mathcal M=\mathrm{SO}(3)\), with \(x,y\) represented by unit quaternions in \(\sS^3\), where \(x\) and \(-x\) represent the same rotation. To preserve this antipodal symmetry, we use hyperbolic-cosine kernel and potential components. Let \(\lambda_R>0\) denote a fixed rotational kernel scale, and define \(\pi_{k,s}(x):=Z_4(f_k+s\lambda_R x/\varepsilon)/[Z_4(f_k+\lambda_R x/\varepsilon)+Z_4(f_k-\lambda_R x/\varepsilon)]\) for \(s\in\{-1,+1\}\).

\vspace{-3mm}

\begin{equation}
    K_\varepsilon(x,y)=\cosh\!\left(\frac{\lambda_R}{\varepsilon}\langle x,y\rangle\right),
    \qquad
    v_{\theta,k}(y)=\cosh(\langle f_k,y\rangle).
    \label{equation:so3_kernel_potential}
\end{equation}

\vspace{-3mm}
\begin{equation}
    r_{\theta,k}(x)\!=\!\frac{1}{2}\!\left[
    Z_4\!\left(f_k\!+\!\frac{\lambda_R}{\varepsilon}x\right)
    \!\!+\!
    Z_4\!\left(f_k\!-\!\frac{\lambda_R}{\varepsilon}x\right)
    \right],
    \quad
    \!\!\!q_{\theta,k}(\cdot|x)
    \!=\!\!\!\!\!\!
    \sum_{s\in\{-1,+1\}}
    \!\!\!\!\!\!\!\pi_{k,s}(x)
    \operatorname{vMF}\!\left(\!f_k\!+\!s\frac{\lambda_R}{\varepsilon}x\!\right).
    \label{equation:so3_normalizer_conditional}
\end{equation}

\vspace{-1mm}

Using the product identity for two hyperbolic cosines yields the closed-form normalizer above. The learnable parameters are \(\theta=\{\alpha_k,f_k\}_{k=1}^{K}\), with \(\alpha_k>0\) and \(f_k\in\mathbb R^4\). The kernel induces the cost \(c(x,y)=-\varepsilon\log\cosh(\lambda_R\langle x,y\rangle/\varepsilon)\). Conditional sampling first draws \(s\) according to \(\pi_{k,s}(x)\), then samples from the corresponding von Mises-Fisher distribution and maps the quaternion to \(\mathrm{SO}(3)\).

\vspace{-2mm}

\paragraph{Special Euclidean group.}
Let \(\mathcal M=\mathrm{SE}(3)\), and write \(x=(x^t,x^R)\), where \(x^t\in\mathbb R^3\) is the translational component and \(x^R\in\sS^3\) is the unit-quaternion representation of the rotational component. We combine an unnormalized Gaussian construction for translations with the \(\mathrm{SO}(3)\) construction above. Let \(K_\varepsilon^t(x^t,y^t):=\exp(-\|x^t-y^t\|_2^2/(2\sigma_\varepsilon^2))\), and let \(K_\varepsilon^R\), \(r_{\theta,k}^R\), and \(q_{\theta,k}^R\) denote the rotational kernel, normalizer, and conditional from the \(\mathrm{SO}(3)\) construction.

\vspace{-3mm}

\begin{equation}
\begin{gathered}
    K_\varepsilon(x,y)=K_\varepsilon^t(x^t,y^t)K_\varepsilon^R(x^R,y^R),\\
    v_{\theta,k}(y)=\exp\!\left[-\frac{1}{2}(y^t-\mu_k)^\top\Sigma_k^{-1}(y^t-\mu_k)\right]\cosh(\langle f_k,y^R\rangle).
\end{gathered}
\label{equation:se3_kernel_potential}
\end{equation}

\begin{equation}
    r_{\theta,k}(x)=r_{\theta,k}^{t}(x^t)r_{\theta,k}^{R}(x^R),
    \qquad
    q_{\theta,k}(y|x)=q_{\theta,k}^{t}(y^t|x^t)q_{\theta,k}^{R}(y^R|x^R).
    \label{equation:se3_normalizer_conditional}
\end{equation}

The translational normalizer is \(r_{\theta,k}^{t}(x^t)=(2\pi)^{3/2}\sigma_\varepsilon^3\det(\Sigma_k)^{1/2}\det(\Sigma_k+\sigma_\varepsilon^2I_3)^{-1/2}\exp[-\frac{1}{2}(x^t-\mu_k)^\top(\Sigma_k+\sigma_\varepsilon^2I_3)^{-1}(x^t-\mu_k)]\). The translational conditional is Gaussian, \(q_{\theta,k}^{t}(\cdot|x^t)=\mathcal N(\cdot|\bar{\mu}_k,A_k)\), with \(A_k=(\sigma_\varepsilon^{-2}I_3+\Sigma_k^{-1})^{-1}\) and \(\bar{\mu}_k=A_k(\sigma_\varepsilon^{-2}x^t+\Sigma_k^{-1}\mu_k)\). The learnable parameters are \(\theta=\{\alpha_k,\mu_k,\Sigma_k,f_k\}_{k=1}^{K}\), with \(\alpha_k>0\), \(\mu_k\in\mathbb R^3\), positive-definite \(\Sigma_k\in\mathbb R^{3\times3}\), and \(f_k\in\mathbb R^4\). The product kernel induces the transport cost \(c(x,y)=\varepsilon\|x^t-y^t\|_2^2/(2\sigma_\varepsilon^2)-\varepsilon\log\cosh(\lambda_R\langle x^R,y^R\rangle/\varepsilon)\). Sampling draws the translational and rotational components independently from \(q_{\theta,k}^{t}\) and \(q_{\theta,k}^{R}\).

\begin{proposition}[Closure under Cartesian products]
\label{proposition:product_parameterization}
Let \(\mathcal M=\prod_{j=1}^{J}\mathcal M_j\), with product measure \(m=\bigotimes_{j=1}^{J}m_j\). Suppose that, for each \(\mathcal M_j\), the kernel \(K_{\varepsilon,j}\) and potential component \(v_{\theta,k,j}\) satisfy the tractability conditions of Proposition~\ref{proposition:parameterization}, with closed-form normalizer \(r_{\theta,k,j}\) and directly sampleable conditional \(q_{\theta,k,j}\). Then, the product parameterization

\vspace{-4mm}

\begin{equation}
    K_\varepsilon(x,y)=\prod_{j=1}^{J}K_{\varepsilon,j}(x^j,y^j),
    \qquad
    v_{\theta,k}(y)=\prod_{j=1}^{J}v_{\theta,k,j}(y^j)
    \label{equation:product_kernel_potential}
\end{equation}

\vspace{-3mm}

also satisfies the conditions of Proposition~\ref{proposition:parameterization}, with

\vspace{-4mm}

\begin{equation}
    r_{\theta,k}(x)=\prod_{j=1}^{J}r_{\theta,k,j}(x^j),
    \qquad
    q_{\theta,k}(y|x)=\prod_{j=1}^{J}q_{\theta,k,j}(y^j|x^j).
    \label{equation:product_normalizer_conditional}
\end{equation}

Therefore, closed-form normalization and direct conditional sampling are preserved.
\end{proposition}

Proposition~\ref{proposition:product_parameterization} allows the parameterizations introduced above to be combined without additional normalization or sampling procedures. In particular, it directly yields tractable parameterizations on product spaces such as \(\mathrm{SO}(3)^J\), \(\mathrm{SE}(3)^J\), and heterogeneous products of the considered manifolds.

We summarize the proposed kernels, potential components, closed-form normalizers, conditional distributions, and induced transport costs in Table~\ref{table:manifold_parameterizations}.

\vspace{-3mm}

\begin{table*}[h]
\centering
\fontsize{7.5pt}{9pt}\selectfont
\renewcommand{\arraystretch}{1.35}
\setlength{\tabcolsep}{5pt}
\caption{Summary of the proposed manifold-aware parameterizations and induced transport costs. For \(\mathrm{SO}(3)\), \(x,y\in\mathbb S^3\) denote unit-quaternion representatives, with antipodal points identified.}
\label{table:manifold_parameterizations}

\makebox[\textwidth][c]{%
\begin{tabular}{
    c
    >{\raggedright\arraybackslash}p{0.28\textwidth}
    >{\raggedright\arraybackslash}p{0.28\textwidth}
    >{\raggedright\arraybackslash}p{0.25\textwidth}
}
\toprule
$\mathcal{M}$
&
\multicolumn{1}{c}{\shortstack{\(K_\varepsilon\)\\[2pt]\(v_{\theta,k}\)}}
&
\multicolumn{1}{c}{\shortstack{\(r_{\theta,k}\)\\[2pt]\(q_{\theta,k}\)}}
&
\multicolumn{1}{c}{\shortstack{Induced cost\\[2pt]\(c(x,y)\)}}
\\
\midrule

\(\sS^{d-1}\)
&
\(\begin{aligned}
&\exp\!\left(\frac{x^\top y}{\varepsilon}\right)\\[3pt]
&\exp(f_k^\top y)
\end{aligned}\)
&
\(\begin{aligned}
&Z_d\!\left(f_k\!+\!\frac{x}{\varepsilon}\right)\\[3pt]
&\operatorname{vMF}\!\left(f_k\!+\!\frac{x}{\varepsilon}\right)
\end{aligned}\)
&
\(
-x^\top y
\)
\\

\midrule

\(\sT^d\)
&
\(\begin{aligned}
&\prod_{j=1}^{d}
       \exp\!\left[\frac{\cos(x^j\!-\!y^j)}{\varepsilon}\right]\\[3pt]
&\prod_{j=1}^{d}
       \exp[a_{kj}\cos y^j\!+\!b_{kj}\sin y^j]
\end{aligned}\)
&
\(\begin{aligned}
&\prod_{j=1}^{d}2\pi I_0(\rho_{kj})\\[3pt]
&\prod_{j=1}^{d}\operatorname{vM}(h_{kj})
\end{aligned}\)
&
\mbox{$
-\sum_{j=1}^{d}
\cos(x^j\!-\!y^j)
$}
\\

\midrule

\(\mathrm{SO}(3)\)
&
\(\begin{aligned}
&\cosh\!\left(
        \frac{\lambda_R}{\varepsilon}
        \langle x,y\rangle
    \right)\\[3pt]
&\cosh(\langle f_k,y\rangle)
\end{aligned}\)
&
\(\begin{aligned}
&\frac{
        Z_4\!\left(f_k\!+\!\frac{\lambda_R}{\varepsilon}x\right)
        \!+\!
        Z_4\!\left(f_k\!-\!\frac{\lambda_R}{\varepsilon}x\right)
    }{2}\\[3pt]
&\sum_{s=\pm1}\pi_{k,s}(x)
    \operatorname{vMF}\!\left(
        f_k\!+\!s\frac{\lambda_R}{\varepsilon}x
    \right)
\end{aligned}\)
&
\(
-\varepsilon
\log\cosh\!\left(
\frac{\lambda_R}{\varepsilon}
\langle x,y\rangle
\right)
\)
\\

\midrule

\(\mathrm{SE}(3)\)
&
\(\begin{aligned}
&K_\varepsilon^t K_\varepsilon^R\\[3pt]
&\exp\!\left[
    -\frac{1}{2}
    (y^t\!-\!\mu_k)^\top
    \!\Sigma_k^{-1}\!
    (y^t\!-\!\mu_k)
    \right]\\
&\quad\times
    \cosh(\langle f_k,y^R\rangle)
\end{aligned}\)
&
\(\begin{aligned}
&r_{\theta,k}^t(x^t)r_{\theta,k}^R(x^R)\\[3pt]
&q_{\theta,k}^t(y^t|x^t)
    q_{\theta,k}^R(y^R|x^R)
\end{aligned}\)
&
\(\begin{aligned}
&
    \frac{\varepsilon}{2\sigma_\varepsilon^2}
    \|x^t\!-\!y^t\|_2^2\\[-1pt]
&\quad
    -\varepsilon
    \log\cosh\!\left(
    \frac{\lambda_R}{\varepsilon}
    \langle x^R,y^R\rangle
    \right)
\end{aligned}\)
\\

\bottomrule
\end{tabular}%
}

\end{table*}

\vspace{-2mm}

\subsection{Training and Sampling}
\label{subsection:training_sampling}

\textbf{Training.} Training the proposed parameterizations requires only optimizing their learnable parameters \(\theta\) using the objective in ~\eqref{equation:lightsb_learning_objective}. Since both \(v_\theta(y)\) and \(r_\theta(x)\) are available in closed form for all the constructions introduced above, the loss can be estimated from independent minibatches \(\{x^{(i)}\}_{i=1}^{B_0}\sim \mu\) and \(\{y^{(j)}\}_{j=1}^{B_1}\sim \nu\) as
\begin{equation}
    \widehat{\mathcal L}(\theta)
    =
    -\frac{1}{B_1}\sum_{j=1}^{B_1}\log v_\theta\!\left(y^{(j)}\right)
    +
    \frac{1}{B_0}\sum_{i=1}^{B_0}\log r_\theta\!\left(x^{(i)}\right).
    \label{equation:monte_carlo_learning_objective}
\end{equation}
This estimate can be optimized using standard gradient-based methods and requires neither paired samples, intermediate trajectory simulation, nor additional inner optimization problems.

\textbf{Sampling.} Given a source point \(x\) and trained parameters \(\theta\), samples from \(q_\theta(\cdot|x)\) are generated using the geometry-specific samplers described in the corresponding constructions of \S \ref{subsection:tractable_parameterizations}. Algorithm~\ref{algorithm:sampling} summarizes the general sampling procedure.

\begin{algorithm}[h]
\caption{Sampling from the learned EOT coupling}
\label{algorithm:sampling}
\begin{algorithmic}[1]
\Require Source point \(x\), trained parameters \(\theta\)
\For{\(k=1,\ldots,K\)}
    \State Compute \(r_{\theta,k}(x)\)
    \State \(\widetilde w_k\gets\alpha_k r_{\theta,k}(x)\)
\EndFor
\State \(w_k\gets\widetilde w_k/\sum_{j=1}^{K}\widetilde w_j\), for \(k=1,\ldots,K\)
\State Sample \(k\sim\operatorname{Categorical}(w_1,\ldots,w_K)\)
\State Sample \(y\sim q_{\theta,k}(\cdot|x)\) using the corresponding manifold-specific sampler
\State \Return \(y\)
\end{algorithmic}
\end{algorithm}

\vspace{-2mm}

\section{Related Works}
\label{section:related_works}
\vspace{-2mm}

In this section, we review two categories of continuous optimal transport methods most relevant to our setting. We first discuss general continuous OT and EOT solvers in \S\ref{subsection:continuous_solvers}, which learn transport maps or coupling representations from samples and support out-of-sample inference. We then consider Riemannian-aware OT solvers in \S\ref{subsection:riemmanian_solvers}, which explicitly incorporate the geometry of the underlying manifold into the formulation.

\vspace{-2mm}

\subsection{Continous OT Solvers}
\label{subsection:continuous_solvers}
\vspace{-1mm}

Continuous neural OT methods provide a flexible framework for learning transport directly from samples. In \citet{seguy2018largescaleoptimaltransportmapping} develop a stochastic optimization approach for large-scale regularized OT, while \citet{daniels2025scorebasedgenerativeneuralnetworks} use score-based models to sample from regularized transport couplings. Neural Optimal Transport (NeuralOT)~\citep{korotin2023neural} introduces a minimax formulation for learning transport maps and plans under strong and weak transport costs. This framework is further generalized by NeuralOT with General Cost Functionals~\citep{asadulaev2024neuraloptimaltransportgeneral}, which accommodates a broader class of transport objectives beyond standard pointwise costs. For entropy-regularized OT, Energy-guided Entropic NOT~\citep{mokrov2024energyguided} parameterizes conditional transport plans using energy-based models and supports general continuous ground costs. For comparison, we consider NeuralOT as a representative general continuous OT baseline.

\vspace{-2mm}

\subsection{Continous Riemmanian OT Solvers}
\label{subsection:riemmanian_solvers}
\vspace{-1mm}

Riemannian Neural Optimal Transport (RNOT)~\citep{micheli2026riemannianneuraloptimaltransport} learns continuous optimal transport maps directly on Riemannian manifolds for the squared geodesic cost. It parameterizes a continuous potential whose \(c\)-transform enforces \(c\)-concavity and constructs the transport map $T(x)$, allowing the learned map to transport previously unseen source samples without discretizing the manifold. Entropic Riemannian Neural Optimal Transport (ERNOT)~\citep{micheli2026entropicriemannianneuraloptimal} solves entropically regularized OT on Riemannian manifolds by learning a neural Schr\"odinger potential that defines the transport coupling. For each source point, this coupling gives a conditional distribution over target points. When a single transported point is required, ERNOT summarizes this distribution using manifold-aware constructions, such as its Riemannian barycenter or the mode of a heat-smoothed version of the conditional distribution.

\section{Experimental Results}
\label{section:experimental_results}
\vspace{-2mm}

In this section, we evaluate our method against the manifold-aware solvers introduced in \S\ref{section:related_works}. We consider synthetic transport problems on \(\sS^{d-1}\), \(\sT^d\), \(\mathrm{SO}(3)\), and \(\mathrm{SE}(3)\), together with real-world applications including continental drift on \(\sS^2\) and crystallographic texture modeling on \(\mathrm{SO}(3)\). Across all experiments, performance is measured using manifold-aware metrics comparing transported source samples with samples from the target distribution. Further details on dataset construction and metric computation are provided in Appendices~\ref{appendix:datasets} and~\ref{appendix:metrics}, respectively. All RNOT and ERNOT experiments use the squared geodesic cost \(c(x,y)=d_{\mathcal M}(x,y)^2\), whereas NeuralOT uses the corresponding manifold-specific induced cost; the full cost expressions are given in Table~\ref{table:manifold_parameterizations}. This protocol tests
how each approach performs as a practical continuous transport solver on the same
source and target distributions, while making the cost choice of each method
explicit.

\textbf{Synthetic Experiments.} For the synthetic experiments, source and target samples are independently drawn from known distributions defined directly on each manifold. We first consider high-dimensional spheres and tori. On \(\sS^{d-1}\), the source is uniform on the sphere and the target is obtained by mapping Gaussian tangent-space samples around a fixed reference point through the exponential map. On \(\sT^d\), the source is uniform in angular coordinates, while the target is obtained from a Gaussian distribution in angular coordinates with periodic wrapping to \([0,2\pi)^d\). We additionally consider setups on \(\mathrm{SO}(3)\) and \(\mathrm{SE}(3)\). On \(\mathrm{SO}(3)\), source and target distributions are generated by mapping Gaussian samples from the Lie algebra through the exponential map, with different center rotations. On \(\mathrm{SE}(3)\), the source combines Haar-uniform rotations with uniformly distributed translations, while the target combines concentrated rotations around a reference orientation with Gaussian translations. We report the squared Maximum Mean Discrepancy (MMD) Table~\ref{table:synthetic_results}. Additional results for high dimensional spheres and tori can be found in Table \ref{table:high_dim_sphere_torus} in Appendix \ref{appendix:additional_results}.

\vspace{-5mm}

\begin{table*}[h]
\centering
\caption{MMD$^2$ results across different manifolds, reported in units of $10^{-2}$. We report mean $\pm$ standard deviation over five seeds. Lower is better.}
\vspace{2mm}
\label{table:synthetic_results}
\resizebox{\textwidth}{!}{
\begin{tabular}{lcccccc}
\toprule
Method & $\sS^{60}$ & $\sS^{100}$ & $\sT^{60}$ & $\sT^{100}$ & $\mathrm{SO}(3)$ & $\mathrm{SE}(3)$ \\
\midrule
ManifoldLightOT (ours)
& $\mathbf{0.07 \pm 0.00}$
& $0.45 \pm 0.05$
& $\mathbf{0.08 \pm 0.00}$
& $\mathbf{0.08 \pm 0.00}$
& $\mathbf{0.03 \pm 0.01}$
& $\mathbf{0.18 \pm 0.01}$ \\

RNOT
& $0.18 \pm 0.04$
& $1.47 \pm 0.21$
& $2.15 \pm 0.20$
& $23.15 \pm 0.57$
& $0.61 \pm 0.27$
& $0.23 \pm 0.04$ \\

ERNOT
& $0.22 \pm 0.01$
& $\mathbf{0.32 \pm 0.09}$
& $0.27 \pm 0.02$
& $0.27 \pm 0.03$
& $0.47 \pm 0.09$
& $5.54 \pm 0.72$ \\

NeuralOT
& $0.27 \pm 0.06$
& $4.38 \pm 8.24$
& $61.59 \pm 0.06$
& $61.61 \pm 0.07$
& $1.27 \pm 2.58$
& $0.27 \pm 0.05$ \\
\bottomrule
\end{tabular}
}
\end{table*}

\textbf{Continental Drift (\(\sS^{2}\)).} Following the continental-drift benchmark introduced by \citet{cohen2021riemannianconvexpotentialmaps} and later considered by RNOT \citep{micheli2026riemannianneuraloptimaltransport}, we construct source and target distributions from continental boundaries at \(150\) million years ago and the present day, respectively, using paleogeographic reconstructions from GPlates \citep{https://doi.org/10.1029/2018GC007584}. Points are sampled uniformly with respect to spherical area on \(\sS^2\) and retained when their longitude-latitude coordinates lie within the corresponding continental polygons, yielding empirical unpaired source and target distributions. The plot of predicted points for ManifoldLightOT appears in Figure \ref{figure:continental_drift}.

\vspace{-1mm}

\begin{figure}[h]
    \centering
    \includegraphics[width=1.0\linewidth]{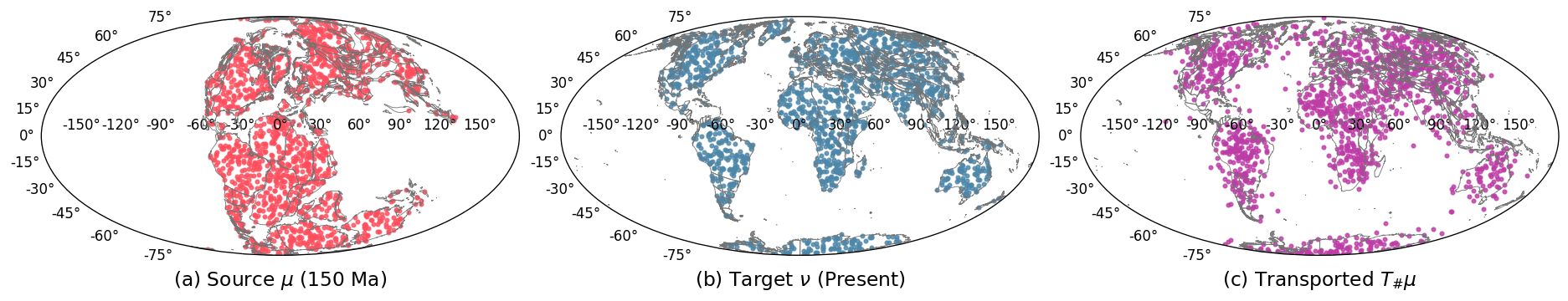}
    \vspace{-5mm}
    \caption{Prediction on continental drift from 150 Ma to the present day for ManifoldLightOT.}

    \label{figure:continental_drift}
\end{figure}

\vspace{-2mm}

\textbf{Crystallographic Texture Transfer.}
Crystallographic texture describes the distribution of crystal orientations within a polycrystalline material and plays an important role in determining its macroscopic mechanical properties. We consider the Ti-6Al-4V electron backscatter diffraction (EBSD) dataset of \citep{zeng2025texture}, which characterizes how crystallographic texture evolves under different hot-rolling conditions. Hot rolling deforms a material by passing it through rollers at elevated temperature, and the reduction percentage denotes the decrease in material thickness during this process. Specifically, we formulate an unpaired transport task from the orientations observed in the starting material to those measured after rolling at \(825^\circ\mathrm{C}\) with \(75\%\) reduction. We consider measurements from both the rolling-direction (RD) and transverse-direction (TD) sections, corresponding respectively to sections aligned with and perpendicular to the rolling direction. We convert the EBSD orientation measurements to unit quaternions; since each orientation represents a three-dimensional rotation, the resulting distributions are naturally supported on \(\mathrm{SO}(3)\). In addition to the MMD metric used in our previous experiments, we evaluate the transported orientations using three crystallography-specific measures. The \emph{normalized ODF difference} compares the predicted and target orientation distribution functions (ODFs), directly measuring agreement between their crystallographic textures, while the \emph{M-index error} compares their texture strengths through the symmetry-aware distribution of crystallographic misorientations. All metrics are computed separately for the RD and TD conditions and they are shown in Table \ref{figure:ebsd}. We also show the  dataset preprocessing and complete metric definitions are provided in Appendix \ref{appendix:datasets} and \ref{appendix:metrics}, respectively.

\vspace{-5mm}

\begin{figure}
    \centering
    \includegraphics[width=1.0\linewidth]{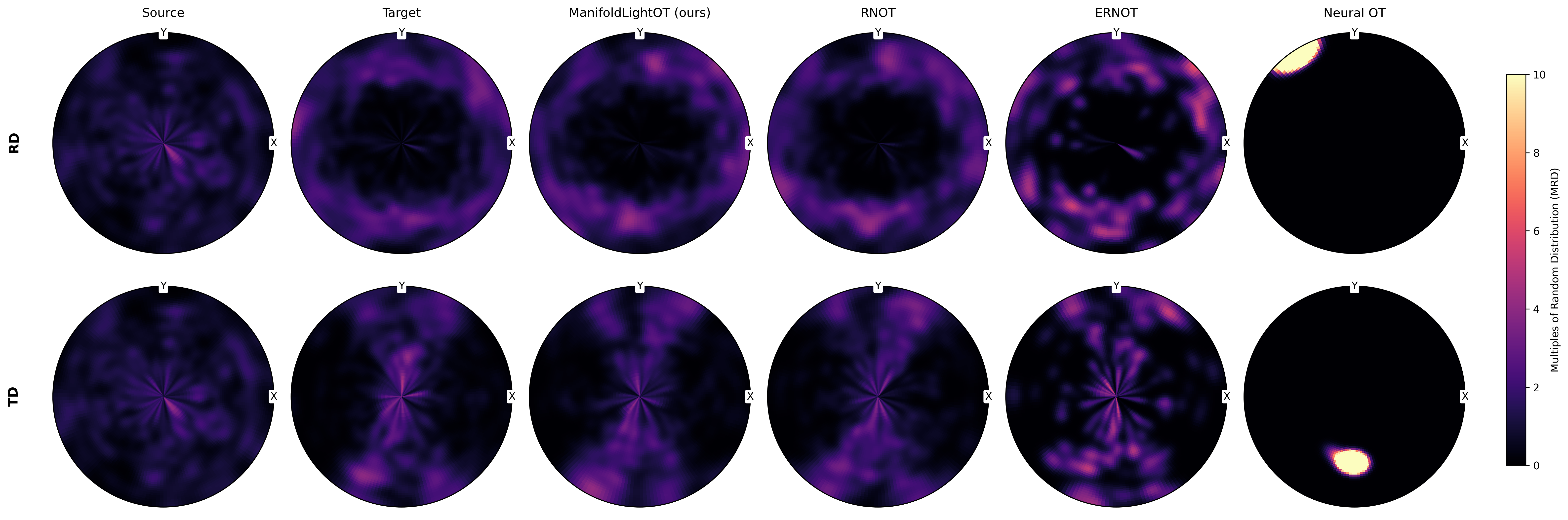}
    \vspace{-3mm}
    \caption{Pole-density plots for the \(\{0001\}\) orientation family in the EBSD texture-transfer task, shown for the RD and TD conditions. For each orientation quaternion \(q\), we take the fixed direction \(h=(0,0,1)\), rotate it into the sample frame, project the resulting unit vector onto the two-dimensional disk using stereographic projection, and estimate the density of the projected points. Colors indicate orientation concentration in multiples of random distribution (MRD).}

    \label{figure:ebsd}
\end{figure}

\begin{table}[h]
\centering
\caption{Results on the crystallographic texture transfer task for the rolling-direction (RD) and transverse-direction (TD) conditions. Lower is better for all metrics.}
\label{tab:ebsd}

\vspace{1mm}

\resizebox{\linewidth}{!}{
\begin{tabular}{lccccc}
\toprule
\makecell[c]{\textbf{Method}} &
\makecell[c]{\textbf{Direction}} &
\makecell[c]{\textbf{MMD$_{\mathrm{SO}(3)}$}} &
\makecell[c]{\textbf{Texture Index}\\\textbf{Error}} &
\makecell[c]{\textbf{Normalized ODF}\\\textbf{Difference}} &
\makecell[c]{\textbf{M-Index Error}} \\
\midrule
\multirow{2}{*}{ManifoldLightOT (ours)}
& RD & \textbf{0.000592} & \textbf{0.01168} & \textbf{0.00270} & \textbf{0.00425} \\
& TD & 0.000830 & \textbf{0.00112} & \textbf{0.00265} & \textbf{0.00883} \\
\midrule
\multirow{2}{*}{RNOT}
& RD & 0.000635 & 0.09022 & 0.01298 & 0.01304 \\
& TD & \textbf{0.000680} & 0.00441 & 0.01936 & 0.02033 \\
\midrule
\multirow{2}{*}{ERNOT}
& RD & 0.000901 & 0.11089 & 0.12242 & 0.01650 \\
& TD & 0.001056 & 0.13406 & 0.18600 & 0.01554 \\
\midrule
\multirow{2}{*}{Neural OT}
& RD & 0.482989 & 37.00789 & 36.40364 & 0.90857 \\
& TD & 0.997347 & 43.39924 & 45.39681 & 0.88277 \\
\bottomrule
\end{tabular}
}
\end{table}

\vspace{-5mm}

\section{Discussion}
\label{section:discussion}

\vspace{-2mm}

ManifoldLightOT combines strong empirical performance with a simpler training and sampling procedure than existing manifold-aware EOT solvers. The experiments show that the proposed kernel-induced EOT formulations adapt well across the considered geometries, including spheres, tori, \(\mathrm{SO}(3)\), and \(\mathrm{SE}(3)\). By using manifold-specific kernel-potential pairs, the method avoids inner optimization procedures and can be trained from Monte Carlo estimates of the learning objective. %

It is important to note that the proposed pipeline, learning objective, and sampling procedure are not tied to a particular manifold. Extending ManifoldLightOT to a new geometry therefore amounts to deriving an appropriate kernel-potential pair that preserves tractable normalization and direct conditional sampling. This derivation is the main manifold-specific challenge of our formulation and can require additional analytical work. At the same time, the induced costs are not specific to ManifoldLightOT: they define valid geometry-aware EOT objectives and could, in principle, be incorporated into other continuous OT/EOT solvers that support general costs. Thus, the proposed kernels provide both a practical light solver and a set of cost functions for manifold-valued transport.

\subsection*{AI use statement}

In this work, generative AI tools were used to assist with polishing the writing, debugging code, and formatting tables and mathematical expressions. The mathematical formulation, derivations, methodological contributions, and experimental results are original to the authors. All AI-assisted content was reviewed and verified by the authors, who take responsibility for the final content of the work.

\subsection*{Ethics Statement}

This work focuses on methodological developments for entropic optimal transport on manifolds and does not involve human subjects, personally identifiable information, or sensitive data. The experiments use synthetic benchmarks and publicly available scientific datasets. We are not aware of any specific ethical risks beyond those associated with the general use of machine learning methods, and we have aimed to report the methodology, experimental setup, and results transparently.

\subsection*{Reproducibility Statement}

To ensure reproducibility, we provide the full implementation of ManifoldLightOT together with the implementations of all baseline methods used in our experiments, as well as the corresponding configuration files for every reported setting. The appendix documents the experimental setup, dataset construction and preprocessing, evaluation metrics, and additional implementation details. Together, these materials are intended to enable reproduction of the reported results and comparisons.

\bibliography{iclr2027_conference}

@inproceedings{
    korotin2023neural,
    title={Neural Optimal Transport},
    author={Korotin, Alexander and Selikhanovych, Daniil and Burnaev, Evgeny},
    booktitle={International Conference on Learning Representations},
    year={2023},
    url={https://openreview.net/forum?id=d8CBRlWNkqH}
}

@book{villani2008optimal,
  title={Optimal transport: old and new},
  author={Villani, C{\'e}dric},
  volume={338},
  year={2008},
  publisher={Springer Science \& Business Media}
}

@article{peyre2019computational,
  title={Computational optimal transport},
  author={Peyr{\'e}, Gabriel and Cuturi, Marco and others},
  journal={Foundations and Trends{\textregistered} in Machine Learning},
  volume={11},
  number={5-6},
  pages={355--607},
  year={2019},
  publisher={Now Publishers, Inc.}
}

@inproceedings{
      mokrov2024energyguided,
      title={Energy-guided Entropic Neural Optimal Transport},
      author={Petr Mokrov and Alexander Korotin and Alexander Kolesov and Nikita Gushchin and Evgeny Burnaev},
      booktitle={The Twelfth International Conference on Learning Representations},
      year={2024},
      url={https://openreview.net/forum?id=d6tUsZeVs7}
}

@inproceedings{gushchin2023entropic,
  title     = {Entropic Neural Optimal Transport via Diffusion Processes},
  author    = {Nikita Gushchin and Alexander Kolesov and Alexander Korotin and Dmitry P. Vetrov and Evgeny Burnaev},
  booktitle = {Thirty-seventh Conference on Neural Information Processing Systems},
  year      = {2023},
  url       = {https://openreview.net/forum?id=fHyLsfMDIs}
}

@article{cuturi2013sinkhorn,
  title={{Sinkhorn} distances: Lightspeed computation of optimal transport},
  author={Cuturi, Marco},
  journal={Advances in neural information processing systems},
  volume={26},
  year={2013}
}

@phdthesis{genevay2019entropy,
  title={Entropy-regularized optimal transport for machine learning},
  author={Genevay, Aude},
  year={2019},
  school={Paris Sciences et Lettres (ComUE)}
}

@book{schrodinger1931umkehrung,
  title={{\"U}ber die Umkehrung der Naturgesetze},
  author={Schr{\"o}dinger, Erwin},
  year={1931},
  publisher={Verlag der Akademie der Wissenschaften in Kommission bei Walter De Gruyter u. Company}
}

@inproceedings{schrodinger1932theorie,
  title={Sur la th{\'e}orie relativiste de l'{\'e}lectron et l'interpr{\'e}tation de la m{\'e}canique quantique},
  author={Schr{\"o}dinger, Erwin},
  booktitle={Annales de l'institut Henri Poincar{\'e}},
  volume={2},
  number={4},
  pages={269--310},
  year={1932}
}

@article{leonard2013survey,
  title={A survey of the {Schr\"odinger} problem and some of its connections with optimal transport},
  author={L{\'e}onard, Christian},
  journal={arXiv preprint arXiv:1308.0215},
  year={2013}
}

@inproceedings{korotin2024light,
  title={Light {Schr{\"o}dinger} Bridge},
  author={Korotin, Alexander and Gushchin, Nikita and Burnaev, Evgeny},
  booktitle={The Twelfth International Conference on Learning Representations},
  year={2024}
}

@inproceedings{dvurechensky2018computational,
  title={Computational optimal transport: Complexity by accelerated gradient descent is better than by {Sinkhorn}’s algorithm},
  author={Dvurechensky, Pavel and Gasnikov, Alexander and Kroshnin, Alexey},
  booktitle={International conference on machine learning},
  pages={1367--1376},
  year={2018},
  organization={PMLR}
}

@article{hutter2021minimax,
  title={Minimax estimation of smooth optimal transport maps},
  author={H{\"u}tter, Jan-Christian and Rigollet, Philippe},
  year={2021}
}

@article{pooladian2021entropic,
  title={Entropic estimation of optimal transport maps},
  author={Pooladian, Aram-Alexandre and Niles-Weed, Jonathan},
  journal={arXiv preprint arXiv:2109.12004},
  year={2021}
}

@inproceedings{makkuva2020optimal,
  title={Optimal transport mapping via input convex neural networks},
  author={Makkuva, Ashok and Taghvaei, Amirhossein and Oh, Sewoong and Lee, Jason},
  booktitle={International Conference on Machine Learning},
  pages={6672--6681},
  year={2020},
  organization={PMLR}
}

@article{deb2021rates,
  title={Rates of estimation of optimal transport maps using plug-in estimators via barycentric projections},
  author={Deb, Nabarun and Ghosal, Promit and Sen, Bodhisattva},
  journal={Advances in Neural Information Processing Systems},
  volume={34},
  pages={29736--29753},
  year={2021}
}

@misc{debortoli2023diffusionschrodingerbridgeapplications,
      title={Diffusion Schr\"odinger Bridge with Applications to Score-Based Generative Modeling}, 
      author={Valentin De Bortoli and James Thornton and Jeremy Heng and Arnaud Doucet},
      year={2023},
      eprint={2106.01357},
      archivePrefix={arXiv},
      primaryClass={stat.ML},
      url={https://arxiv.org/abs/2106.01357}, 
}

@misc{tong2024simulationfreeschrodingerbridgesscore,
      title={Simulation-free Schr\"odinger bridges via score and flow matching}, 
      author={Alexander Tong and Nikolay Malkin and Kilian Fatras and Lazar Atanackovic and Yanlei Zhang and Guillaume Huguet and Guy Wolf and Yoshua Bengio},
      year={2024},
      eprint={2307.03672},
      archivePrefix={arXiv},
      primaryClass={cs.LG},
      url={https://arxiv.org/abs/2307.03672}, 
}

@misc{micheli2026entropicriemannianneuraloptimal,
      title={Entropic Riemannian Neural Optimal Transport}, 
      author={Alessandro Micheli and Silvia Sapora and Anthea Monod and Samir Bhatt},
      year={2026},
      eprint={2605.04255},
      archivePrefix={arXiv},
      primaryClass={stat.ML},
      url={https://arxiv.org/abs/2605.04255}, 
}

@misc{chen2024flowmatchinggeneralgeometries,
      title={Flow Matching on General Geometries}, 
      author={Ricky T. Q. Chen and Yaron Lipman},
      year={2024},
      eprint={2302.03660},
      archivePrefix={arXiv},
      primaryClass={cs.LG},
      url={https://arxiv.org/abs/2302.03660}, 
}

@misc{mathieu2020riemanniancontinuousnormalizingflows,
      title={Riemannian Continuous Normalizing Flows}, 
      author={Emile Mathieu and Maximilian Nickel},
      year={2020},
      eprint={2006.10605},
      archivePrefix={arXiv},
      primaryClass={stat.ML},
      url={https://arxiv.org/abs/2006.10605}, 
}

@misc{debortoli2022riemannianscorebasedgenerativemodelling,
      title={Riemannian Score-Based Generative Modelling}, 
      author={Valentin De Bortoli and Emile Mathieu and Michael Hutchinson and James Thornton and Yee Whye Teh and Arnaud Doucet},
      year={2022},
      eprint={2202.02763},
      archivePrefix={arXiv},
      primaryClass={cs.LG},
      url={https://arxiv.org/abs/2202.02763}, 
}

@misc{micheli2026riemannianneuraloptimaltransport,
      title={Riemannian Neural Optimal Transport}, 
      author={Alessandro Micheli and Yueqi Cao and Anthea Monod and Samir Bhatt},
      year={2026},
      eprint={2602.03566},
      archivePrefix={arXiv},
      primaryClass={cs.LG},
      url={https://arxiv.org/abs/2602.03566}, 
}

@misc{cohen2021riemannianconvexpotentialmaps,
      title={Riemannian Convex Potential Maps}, 
      author={Samuel Cohen and Brandon Amos and Yaron Lipman},
      year={2021},
      eprint={2106.10272},
      archivePrefix={arXiv},
      primaryClass={cs.LG},
      url={https://arxiv.org/abs/2106.10272}, 
}

@article{https://doi.org/10.1029/2018GC007584,
author = {Müller, R. Dietmar and Cannon, John and Qin, Xiaodong and Watson, Robin J. and Gurnis, Michael and Williams, Simon and Pfaffelmoser, Tobias and Seton, Maria and Russell, Samuel H. J. and Zahirovic, Sabin},
title = {GPlates: Building a Virtual Earth Through Deep Time},
journal = {Geochemistry, Geophysics, Geosystems},
volume = {19},
number = {7},
pages = {2243-2261},
doi = {https://doi.org/10.1029/2018GC007584},
url = {https://agupubs.onlinelibrary.wiley.com/doi/abs/10.1029/2018GC007584},
eprint = {https://agupubs.onlinelibrary.wiley.com/doi/pdf/10.1029/2018GC007584},
year = {2018}
}

@misc{gushchin2023buildingbridgeschrodingercontinuous,
      title={Building the Bridge of Schr\"odinger: A Continuous Entropic Optimal Transport Benchmark}, 
      author={Nikita Gushchin and Alexander Kolesov and Petr Mokrov and Polina Karpikova and Andrey Spiridonov and Evgeny Burnaev and Alexander Korotin},
      year={2023},
      eprint={2306.10161},
      archivePrefix={arXiv},
      primaryClass={cs.LG},
      url={https://arxiv.org/abs/2306.10161}, 
}

@misc{korotin2021neuraloptimaltransportsolvers,
      title={Do Neural Optimal Transport Solvers Work? A Continuous Wasserstein-2 Benchmark}, 
      author={Alexander Korotin and Lingxiao Li and Aude Genevay and Justin Solomon and Alexander Filippov and Evgeny Burnaev},
      year={2021},
      eprint={2106.01954},
      archivePrefix={arXiv},
      primaryClass={cs.LG},
      url={https://arxiv.org/abs/2106.01954}, 
}

@misc{persiianov2026learningpopulationdynamicsinverse,
      title={Learning of Population Dynamics: Inverse Optimization Meets JKO Scheme}, 
      author={Mikhail Persiianov and Jiawei Chen and Petr Mokrov and Alexander Tyurin and Evgeny Burnaev and Alexander Korotin},
      year={2026},
      eprint={2506.01502},
      archivePrefix={arXiv},
      primaryClass={cs.LG},
      url={https://arxiv.org/abs/2506.01502}, 
}

@article{zeng2025texture,
  title={Effect of Temperature, Strain and Microstructure on Texture Evolution of Ti-6Al-4V During Hot Rolling},
  author={Zeng, Xiaohan and Daniel, Christopher S. and Asheshov, William and Michalik, Stefan and Prangnell, Philip and Quinta da Fonseca, Jo{\~a}o},
  journal={Metallurgical and Materials Transactions A},
  volume={56},
  number={10},
  pages={4620--4643},
  year={2025},
  doi={10.1007/s11661-025-07929-8}
}

@misc{asadulaev2024neuraloptimaltransportgeneral,
      title={Neural Optimal Transport with General Cost Functionals}, 
      author={Arip Asadulaev and Alexander Korotin and Vage Egiazarian and Petr Mokrov and Evgeny Burnaev},
      year={2024},
      eprint={2205.15403},
      archivePrefix={arXiv},
      primaryClass={cs.LG},
      url={https://arxiv.org/abs/2205.15403}, 
}

@misc{manole2024pluginestimationsmoothoptimal,
      title={Plugin Estimation of Smooth Optimal Transport Maps}, 
      author={Tudor Manole and Sivaraman Balakrishnan and Jonathan Niles-Weed and Larry Wasserman},
      year={2024},
      eprint={2107.12364},
      archivePrefix={arXiv},
      primaryClass={math.ST},
      url={https://arxiv.org/abs/2107.12364}, 
}

@misc{seguy2018largescaleoptimaltransportmapping,
      title={Large-Scale Optimal Transport and Mapping Estimation}, 
      author={Vivien Seguy and Bharath Bhushan Damodaran and Rémi Flamary and Nicolas Courty and Antoine Rolet and Mathieu Blondel},
      year={2018},
      eprint={1711.02283},
      archivePrefix={arXiv},
      primaryClass={stat.ML},
      url={https://arxiv.org/abs/1711.02283}, 
}

@misc{daniels2025scorebasedgenerativeneuralnetworks,
      title={Score-based Generative Neural Networks for Large-Scale Optimal Transport}, 
      author={Mara Daniels and Tyler Maunu and Paul Hand},
      year={2025},
      eprint={2110.03237},
      archivePrefix={arXiv},
      primaryClass={cs.LG},
      url={https://arxiv.org/abs/2110.03237}, 
}

@article{SKEMER2005157,
title = {The misorientation index: Development of a new method for calculating the strength of lattice-preferred orientation},
journal = {Tectonophysics},
volume = {411},
number = {1},
pages = {157-167},
year = {2005},
issn = {0040-1951},
doi = {https://doi.org/10.1016/j.tecto.2005.08.023},
url = {https://www.sciencedirect.com/science/article/pii/S0040195105004439},
author = {Philip Skemer and Ikuo Katayama and Zhenting Jiang and Shun-ichiro Karato},
}

@ARTICLE{ottseismic2024,
       author = {{Ott}, Jason N. and {Condit}, Cailey B. and {Schulte-Pelkum}, Vera and {Bernard}, Rachel and {Pec}, Matej},
        title = "{Seismic Anisotropy of Mafic Blueschists: EBSD-Based Constraints From the Exhumed Rock Record}",
      journal = {Journal of Geophysical Research (Solid Earth)},
         year = 2024,
        month = feb,
       volume = {129},
       number = {2},
          eid = {e2023JB027679},
        pages = {e2023JB027679},
          doi = {10.1029/2023JB027679},
       adsurl = {https://ui.adsabs.harvard.edu/abs/2024JGRB..12927679O}
}

@article{VANHOUTTE2005589,
title = {Deformation texture prediction: from the Taylor model to the advanced Lamel model},
journal = {International Journal of Plasticity},
volume = {21},
number = {3},
pages = {589-624},
year = {2005},
issn = {0749-6419},
doi = {https://doi.org/10.1016/j.ijplas.2004.04.011},
url = {https://www.sciencedirect.com/science/article/pii/S0749641904000610},
author = {Paul {Van Houtte} and Saiyi Li and Marc Seefeldt and Laurent Delannay},
}
\bibliographystyle{iclr2027_conference}

\newpage
\appendix
\section{Appendix}

\subsection{Proofs}
\label{appendix:proofs}

\textbf{Proof of Proposition~\ref{proposition:parameterization}}

\begin{proof}
By definition of the mixture potential,
\[
v_\theta(y)
=
\sum_{k=1}^{K}\alpha_k v_{\theta,k}(y).
\]
Therefore, using linearity of integration, its normalization function is
\[
\begin{aligned}
r_\theta(x)
&=
\int_{\mathcal M}
K_\varepsilon(x,y)v_\theta(y)\,dm(y)\\
&=
\int_{\mathcal M}
K_\varepsilon(x,y)
\sum_{k=1}^{K}\alpha_k v_{\theta,k}(y)\,dm(y)\\
&=
\sum_{k=1}^{K}
\alpha_k
\int_{\mathcal M}
K_\varepsilon(x,y)v_{\theta,k}(y)\,dm(y)\\
&=
\sum_{k=1}^{K}\alpha_k r_{\theta,k}(x).
\end{aligned}
\]

The corresponding conditional distribution is
\[
\begin{aligned}
q_\theta(y|x)
&=
\frac{K_\varepsilon(x,y)v_\theta(y)}
     {r_\theta(x)}\\
&=
\sum_{k=1}^{K}
\frac{\alpha_k r_{\theta,k}(x)}
     {r_\theta(x)}
\frac{K_\varepsilon(x,y)v_{\theta,k}(y)}
     {r_{\theta,k}(x)}\\
&=
\sum_{k=1}^{K}
w_{\theta,k}(x)\,
q_{\theta,k}(y|x),
\end{aligned}
\]
where
\[
w_{\theta,k}(x)
=
\frac{\alpha_k r_{\theta,k}(x)}
     {r_\theta(x)}.
\]
Since all \(\alpha_k\) and \(r_{\theta,k}(x)\) are positive,
\(w_{\theta,k}(x)>0\), and
\[
\sum_{k=1}^{K}w_{\theta,k}(x)=1.
\]
Hence \(q_\theta(\cdot|x)\) is a finite mixture of the component conditionals
\(q_{\theta,k}(\cdot|x)\). Since each component can be sampled directly, one
can sample from \(q_\theta(\cdot|x)\) by first drawing
\(k\sim\operatorname{Categorical}(w_{\theta,1}(x),\ldots,w_{\theta,K}(x))\)
and then sampling from \(q_{\theta,k}(\cdot|x)\). Thus, closed-form
normalization and direct conditional sampling are preserved by the mixture.
\end{proof}

\textbf{Proof of Proposition~\ref{proposition:product_parameterization}}

\begin{proof}
Using the product structure of both the kernel and potential component,
\[
K_\varepsilon(x,y)v_{\theta,k}(y)
=
\prod_{j=1}^{J}
K_{\varepsilon,j}(x^j,y^j)
v_{\theta,k,j}(y^j).
\]
Since the reference measure also factorizes as
\(m=\bigotimes_{j=1}^{J}m_j\), the corresponding normalization function satisfies
\[
\begin{aligned}
r_{\theta,k}(x)
&=
\int_{\mathcal M}
K_\varepsilon(x,y)v_{\theta,k}(y)\,dm(y)\\
&=
\int_{\prod_{j=1}^{J}\mathcal M_j}
\prod_{j=1}^{J}
K_{\varepsilon,j}(x^j,y^j)
v_{\theta,k,j}(y^j)
\prod_{j=1}^{J}dm_j(y^j)\\
&=
\prod_{j=1}^{J}
\int_{\mathcal M_j}
K_{\varepsilon,j}(x^j,y^j)
v_{\theta,k,j}(y^j)\,dm_j(y^j)\\
&=
\prod_{j=1}^{J}
r_{\theta,k,j}(x^j).
\end{aligned}
\]

Therefore, the normalized conditional distribution becomes
\[
\begin{aligned}
q_{\theta,k}(y|x)
&=
\frac{
K_\varepsilon(x,y)v_{\theta,k}(y)
}{
r_{\theta,k}(x)
}\\
&=
\prod_{j=1}^{J}
\frac{
K_{\varepsilon,j}(x^j,y^j)
v_{\theta,k,j}(y^j)
}{
r_{\theta,k,j}(x^j)
}\\
&=
\prod_{j=1}^{J}
q_{\theta,k,j}(y^j|x^j).
\end{aligned}
\]

Hence the conditional factorizes across the component manifolds. Since each
\(q_{\theta,k,j}(\cdot|x^j)\) can be sampled directly by assumption, a sample
from \(q_{\theta,k}(\cdot|x)\) is obtained by independently sampling
\(y^j\sim q_{\theta,k,j}(\cdot|x^j)\) for \(j=1,\ldots,J\) and setting
\(y=(y^1,\ldots,y^J)\). Thus, closed-form normalization and direct conditional
sampling are preserved under Cartesian products.
\end{proof}

\textbf{Proofs of the manifold-specific parameterizations}

\paragraph{Distributional notation.}
For \(g\in\mathbb R^d\), we define
\[
Z_d(g)
=
\int_{\mathbb S^{d-1}}
\exp(g^\top y)\,dm(y),
\]
and denote by \(\operatorname{vMF}(g)\) the von Mises-Fisher distribution on
\(\mathbb S^{d-1}\) with density
\(\exp(g^\top y)/Z_d(g)\).
Since \(Z_d(g)\) depends only on \(\|g\|_2\), we have
\(Z_d(-g)=Z_d(g)\).
For the torus construction, \(\operatorname{vM}(h)\), with
\(h\in\mathbb R^2\), denotes the circular von Mises distribution whose density
is proportional to
\(\exp(h^\top(\cos y,\sin y)^\top)\); its normalizing constant is
\(2\pi I_0(\|h\|_2)\), where \(I_0\) is the modified Bessel function of the
first kind of order zero, defined by
\[
I_0(t)
=
\frac{1}{2\pi}\int_0^{2\pi}\exp(t\cos\phi)\,d\phi.
\]

\paragraph{Sphere.}
Consider one potential component \(v_{\theta,k}\). Multiplying it by the
kernel in \eqref{equation:sphere_kernel_potential} gives
\[
K_\varepsilon(x,y)v_{\theta,k}(y)
=
\exp\!\left(\frac{x^\top y}{\varepsilon}\right)
\exp(f_k^\top y)
=
\exp\!\left(\left(f_k+\frac{x}{\varepsilon}\right)^\top y\right).
\]
Thus, the kernel-potential product is itself an unnormalized von
Mises-Fisher density with natural parameter
\(f_k+x/\varepsilon\). Integrating over the sphere immediately gives
\(r_{\theta,k}(x)=Z_d(f_k+x/\varepsilon)\), and dividing by this
normalizer gives
\(q_{\theta,k}(\cdot|x)=\operatorname{vMF}(f_k+x/\varepsilon)\).
Hence each component satisfies the assumptions of
Proposition~\ref{proposition:parameterization}, which yields closed-form
normalization and direct sampling for the complete mixture \(q_\theta\).

\paragraph{Torus.}
For a fixed coordinate \(j\), the addition identity
\(\cos(x^j-y^j)=\cos x^j\cos y^j+\sin x^j\sin y^j\) gives
\[
\begin{aligned}
&\frac{1}{\varepsilon}\cos(x^j-y^j)
+a_{kj}\cos y^j+b_{kj}\sin y^j\\
&\qquad=
\left(a_{kj}+\frac{1}{\varepsilon}\cos x^j\right)\cos y^j
+
\left(b_{kj}+\frac{1}{\varepsilon}\sin x^j\right)\sin y^j\\
&\qquad=
h_{kj}(x)^\top
\begin{pmatrix}
\cos y^j\\
\sin y^j
\end{pmatrix}.
\end{aligned}
\]
Therefore,
\[
K_\varepsilon(x,y)v_{\theta,k}(y)
=
\prod_{j=1}^{d}
\exp\!\left[
h_{kj}(x)^\top
\begin{pmatrix}
\cos y^j\\
\sin y^j
\end{pmatrix}
\right].
\]
The product factorizes across coordinates. Since the normalizer of each factor
is \(2\pi I_0(\|h_{kj}(x)\|_2)=2\pi I_0(\rho_{kj}(x))\), integration gives
\[
r_{\theta,k}(x)
=
\prod_{j=1}^{d}
2\pi I_0(\rho_{kj}(x)),
\qquad
q_{\theta,k}(\cdot|x)
=
\prod_{j=1}^{d}\operatorname{vM}(h_{kj}(x)).
\]
Thus the coordinates can be sampled independently, and
Proposition~\ref{proposition:parameterization} gives the corresponding
mixture construction.

\paragraph{Rotation Group.}
Let
\[
g_{k,s}(x)=f_k+s\frac{\lambda_R}{\varepsilon}x,
\qquad s\in\{-1,+1\}.
\]
Using
\(\cosh(a)\cosh(b)
=\frac{1}{2}[\cosh(a+b)+\cosh(a-b)]\), we obtain
\[
K_\varepsilon(x,y)v_{\theta,k}(y)
=
\frac{1}{2}
\sum_{s\in\{-1,+1\}}
\cosh\!\left(g_{k,s}(x)^\top y\right).
\]
For any \(g\in\mathbb R^4\),
\[
\int_{\mathbb S^3}\cosh(g^\top y)\,dm(y)
=
\frac{Z_4(g)+Z_4(-g)}{2}
=
Z_4(g).
\]
Consequently,
\[
r_{\theta,k}(x)
=
\frac{1}{2}
\left[
Z_4\!\left(f_k+\frac{\lambda_R}{\varepsilon}x\right)
+
Z_4\!\left(f_k-\frac{\lambda_R}{\varepsilon}x\right)
\right].
\]
After normalization, the two terms are selected with probabilities
\[
\pi_{k,s}(x)
=
\frac{
Z_4\!\left(f_k+s\frac{\lambda_R}{\varepsilon}x\right)
}{
Z_4\!\left(f_k+\frac{\lambda_R}{\varepsilon}x\right)
+
Z_4\!\left(f_k-\frac{\lambda_R}{\varepsilon}x\right)
}.
\]
Recall that the unit quaternions \(y\) and \(-y\) represent the same rotation.
Accordingly, a density proportional to \(\cosh(g^\top y)\) on
\(\mathrm{SO}(3)\) can be sampled by drawing a quaternion from
\(\operatorname{vMF}(g)\) on \(\mathbb S^3\) and then mapping it to its
corresponding rotation. Hence
\[
q_{\theta,k}(\cdot|x)
=
\sum_{s\in\{-1,+1\}}
\pi_{k,s}(x)\,
\operatorname{vMF}\!\left(f_k+s\frac{\lambda_R}{\varepsilon}x\right),
\]
where the sampled quaternion is understood to be mapped to
\(\mathrm{SO}(3)\). Each component is therefore directly sampleable, and
Proposition~\ref{proposition:parameterization} gives the complete mixture.

\paragraph{Special Euclidean Group.}
We use the product representation
\(\mathrm{SE}(3)\simeq\mathbb R^3\times\mathrm{SO}(3)\).
The rotational factor has already been shown above to satisfy the tractability
conditions, so it remains to verify the translational factor. For
\(K_\varepsilon^t\) and the translational part of \(v_{\theta,k}\), define
\[
A_k
=
\left(\sigma_\varepsilon^{-2}I_3+\Sigma_k^{-1}\right)^{-1},
\qquad
\bar{\mu}_k
=
A_k
\left(
\sigma_\varepsilon^{-2}x^t+\Sigma_k^{-1}\mu_k
\right).
\]
Completing the square in \(y^t\) gives
\[
\begin{aligned}
&\frac{\|x^t-y^t\|_2^2}{\sigma_\varepsilon^2}
+
(y^t-\mu_k)^\top\Sigma_k^{-1}(y^t-\mu_k)\\
&\qquad=
(y^t-\bar{\mu}_k)^\top A_k^{-1}(y^t-\bar{\mu}_k)
+
(x^t-\mu_k)^\top
(\Sigma_k+\sigma_\varepsilon^2I_3)^{-1}
(x^t-\mu_k).
\end{aligned}
\]
The dependence on \(y^t\) is therefore Gaussian, so the normalized conditional
is
\(q_{\theta,k}^{t}(\cdot|x^t)
=\mathcal N(\cdot|\bar{\mu}_k,A_k)\).
Integrating the Gaussian term yields
\[
\begin{aligned}
r_{\theta,k}^{t}(x^t)
&=
(2\pi)^{3/2}\sigma_\varepsilon^3
\det(\Sigma_k)^{1/2}
\det(\Sigma_k+\sigma_\varepsilon^2I_3)^{-1/2}\\
&\quad\times
\exp\!\left[
-\frac{1}{2}
(x^t-\mu_k)^\top
(\Sigma_k+\sigma_\varepsilon^2I_3)^{-1}
(x^t-\mu_k)
\right].
\end{aligned}
\]
Hence both the translational and rotational factors satisfy the assumptions of
Proposition~\ref{proposition:parameterization}. Applying
Proposition~\ref{proposition:product_parameterization} gives
\(r_{\theta,k}(x)=r_{\theta,k}^{t}(x^t)r_{\theta,k}^{R}(x^R)\) and
\(q_{\theta,k}(y|x)
=q_{\theta,k}^{t}(y^t|x^t)q_{\theta,k}^{R}(y^R|x^R)\).
Finally, Proposition~\ref{proposition:parameterization} gives the full
\(K\)-component mixture.

\subsection{Datasets}
\label{appendix:datasets}

\subsubsection{Synthetic Datasets}
\label{subsubsection:synthetic_datasets}

For all synthetic experiments, we independently generate \(20000\) training samples and \(5120\) test samples for each of the source and target marginals. Training uses only the unpaired source and target samples, while the held-out test sets are used for evaluation.

\textbf{Synthetic Sphere.}
For the experiments on \(\sS^{D}\subset\mathbb R^{D+1}\), the source distribution is uniform on the sphere. A source sample is generated by drawing \(z\sim\mathcal N(0,I_{D+1})\) and normalizing it as \(x=z/\|z\|_2\). The target is a wrapped normal distribution centered at the fixed point \(\ell=(-1,0,\ldots,0)\in\sS^D\). We first draw \(\xi\sim\mathcal N(0,\sigma^2 I_D)\) and embed it in \(\mathbb R^{D+1}\) as \(\widetilde{\xi}=(0,\xi)\). Since \(\ell^\top\widetilde{\xi}=0\), this vector is tangent to the sphere at \(\ell\), i.e., it points along the local \(D\)-dimensional plane orthogonal to \(\ell\). The sample is then mapped onto the sphere using the spherical exponential map,
\[
y=
\cos(\|\widetilde{\xi}\|_2)\ell
+
\frac{\sin(\|\widetilde{\xi}\|_2)}{\|\widetilde{\xi}\|_2}\widetilde{\xi},
\]
where \(y=\ell\) when \(\|\widetilde{\xi}\|_2=0\). We use \(\sigma=0.3\) for all dimensions.

\textbf{Synthetic Torus.}
A point on \(\sT^D\) is represented by \(D\) angular coordinates \(\theta=(\theta^1,\ldots,\theta^D)\), with each \(\theta^j\in[0,2\pi)\). Source samples are generated independently as \(\theta^j\sim\operatorname{Unif}[0,2\pi)\). For the target, we use a wrapped Gaussian centered at \((\pi,\ldots,\pi)\). Specifically, we draw \(z\sim\mathcal N(0,\sigma^2I_D)\) and set \(\theta=(\pi\mathbf 1+z)\bmod 2\pi\), where the modulo operation is applied coordinate-wise and \(\mathbf 1\) denotes the \(D\)-dimensional vector of ones. We use \(\sigma=0.3\) in all experiments. When an ambient Euclidean representation is required, each angle is mapped to the corresponding point on the unit circle as \(\theta^j\mapsto(\cos\theta^j,\sin\theta^j)\).

\textbf{Synthetic \(\mathrm{SO}(3)\).}
Rotations are represented by unit quaternions \(q=(q_0,q_1,q_2,q_3)\in\sS^3\), with \(q\) and \(-q\) representing the same three-dimensional rotation. Both marginals are wrapped normal distributions with different center rotations. The source is centered at the identity rotation, represented by \(q_{\mathrm{s}}=(1,0,0,0)\). The target center is obtained by rotating through an angle \(\phi=2.5\) radians around the axis \(a=(0,0,1)\), giving the unit quaternion \(q_{\mathrm{t}}=(\cos(\phi/2),\sin(\phi/2)a)\).

To sample around either center \(q_c\), we draw a three-dimensional perturbation \(v\sim\mathcal N(0,\sigma^2I_3)\), with \(\sigma=0.8\). Writing \(r=\|v\|_2\), this perturbation is converted into the rotation quaternion
\[
\delta(v)
=
\left(
\cos\frac{r}{2},
\frac{\sin(r/2)}{r}v
\right),
\]
with \(\delta(0)=(1,0,0,0)\). The final sample is \(q=q_c\otimes\delta(v)\), where \(\otimes\) denotes quaternion multiplication. Thus, Gaussian perturbations in three-dimensional rotation-vector coordinates are wrapped around the corresponding center rotation. Source and target train and test samples are generated independently.

\textbf{Synthetic \(\mathrm{SE}(3)\).}
A point in \(\mathrm{SE}(3)\) consists of a three-dimensional rotation and a three-dimensional translation, and is represented in our implementation by a unit quaternion \(q\in\sS^3\) together with a translation vector \(t\in\mathbb R^3\). For the source distribution, rotations are sampled uniformly over all three-dimensional rotations, i.e., from the Haar-uniform distribution on \(\mathrm{SO}(3)\), while translations are sampled uniformly from the cube \([-4,4]^3\).

For the target rotation, we use a wrapped normal distribution centered at a rotation of \(60^\circ\) around the \(z\)-axis. As in the \(\mathrm{SO}(3)\) experiment, we draw a rotation vector \(v\sim\mathcal N(0,0.3^2I_3)\), convert it into a rotation quaternion \(\delta(v)\), and compose it with the target center rotation. The target translations follow a Gaussian distribution with mean \(\mu_t=(1,0.5,-0.5)\) and isotropic standard deviation \(0.5\). To keep the translational support identical to that of the source, samples outside \([-4,4]^3\) are rejected and resampled. Rotation and translation are sampled independently within each marginal. Each resulting sample is stored as \((q_0,q_1,q_2,q_3,t_1,t_2,t_3)\).

\subsubsection{Continental Drift}

\textbf{Continental Drift.}
We follow the continental-drift benchmark of \citet{cohen2021riemannianconvexpotentialmaps} using paleogeographic reconstructions from GPlates \citep{https://doi.org/10.1029/2018GC007584}. The source distribution corresponds to continental regions at \(150\) million years ago, while the target corresponds to present-day continental regions. The reconstructions are provided as GeoJSON files containing polygonal continental boundaries in geographic coordinates, with each polygon represented by a sequence of longitude-latitude vertices. Thus, the raw data do not consist of points on \(\sS^2\); instead, these polygons define the continental regions from which spherical samples are obtained.

Points are sampled uniformly with respect to surface area on \(\sS^2\) and restricted to the corresponding continental regions by rejection sampling. Specifically, we draw a longitude \(\lambda\sim\operatorname{Unif}[-\pi,\pi)\) and a scalar \(z\sim\operatorname{Unif}[-1,1]\). The latitude is then \(\varphi=\arcsin z\), and the corresponding point on the unit sphere is
\[
x=
(\cos\varphi\cos\lambda,\,
 \cos\varphi\sin\lambda,\,
 \sin\varphi).
\]
This construction is uniform with respect to spherical surface area. The longitude-latitude coordinates of each candidate point are tested against the continental polygons for the corresponding time period, and candidates outside all polygons are rejected. Accepted points are represented by their three-dimensional unit vectors \(x\in\sS^2\) for training and evaluation. The procedure is applied independently to the \(150\) Ma and present-day polygons, yielding unpaired empirical source and target distributions. We sample \(50{,}000\) training points from each distribution.

\subsubsection{Crystallographic Texture Dataset}
\label{subsubsection:ebsd_dataset}

We use the Ti-6Al-4V electron backscatter diffraction (EBSD) dataset of \citet{zeng2025texture}, which contains crystallographic orientation maps for the starting material and for samples subjected to different hot-rolling conditions. The raw measurements are provided as EBSD maps in the DefDAP format, with each spatial measurement associated with a crystallographic orientation and a phase label. In our experiments, we use the \emph{starting material} as the source state and the material rolled at \(825^\circ\mathrm{C}\) with a \(75\%\) thickness reduction as the target state. We construct two transport tasks corresponding to the rolling-direction (RD) and transverse-direction (TD) sections: for each section, the source distribution is formed by the crystallographic orientations measured in the starting material, while the target distribution is formed by the orientations measured after the \(825^\circ\mathrm{C}\), \(75\%\)-reduction treatment. These data are naturally unpaired because the source and target orientations are obtained from separate EBSD maps acquired at different material states; consequently, an orientation measured at a given spatial location in the starting material has no one-to-one correspondence with an orientation in the rolled sample. The learning problem is therefore to transport between the two empirical orientation distributions rather than between matched orientation pairs.

We preprocess each EBSD map using \texttt{DefDAP}. Given a raw \texttt{.cpr} file, we load the corresponding EBSD map and flatten its orientation and phase arrays so that each measured point is represented independently. Each crystallographic orientation is converted to its quaternion coefficients \(q=(q_0,q_1,q_2,q_3)\), after which the quaternion is normalized to unit norm. Since \(q\) and \(-q\) represent the same element of \(\mathrm{SO}(3)\), we use a canonical representative by changing the sign whenever \(q_0<0\), so that all stored quaternions satisfy \(q_0\geq0\). The resulting unit quaternions, together with their corresponding phase labels, are saved as PyTorch tensors for subsequent training and evaluation.

\subsection{Metrics}
\label{appendix:metrics}

\paragraph{Maximum Mean Discrepancy.}
For all manifold experiments, we compare the transported distribution with the target distribution using the squared Maximum Mean Discrepancy (MMD) equipped with a kernel adapted to the corresponding manifold. Let \(P\) and \(Q\) denote two probability distributions on \(\mathcal M\), and let \(k:\mathcal M\times\mathcal M\rightarrow\mathbb R\) be a positive-definite kernel. The squared MMD is defined as
\begin{equation}
\operatorname{MMD}^2(P,Q;k)
=
\mathbb E_{x,x'\sim P}[k(x,x')]
+
\mathbb E_{y,y'\sim Q}[k(y,y')]
-
2\mathbb E_{x\sim P,y\sim Q}[k(x,y)],
\end{equation}
where \(x,x'\) and \(y,y'\) are independent samples from their respective distributions. Given target samples \(\{y_i\}_{i=1}^{N}\) and transported samples \(\{\widehat y_j\}_{j=1}^{M}\), we use the empirical estimator
\begin{equation}
\widehat{\operatorname{MMD}}^2
=
\frac{1}{N^2}\sum_{i,i'=1}^{N}k(y_i,y_{i'})
+
\frac{1}{M^2}\sum_{j,j'=1}^{M}k(\widehat y_j,\widehat y_{j'})
-
\frac{2}{NM}\sum_{i=1}^{N}\sum_{j=1}^{M}k(y_i,\widehat y_j).
\label{equation:empirical_mmd}
\end{equation}
This is the biased empirical estimator, with the diagonal terms included, and corresponds directly to the quantity computed in our implementation. Lower values indicate closer agreement between the transported and target distributions. We use different kernels for each manifold in order to respect its geometry.

\paragraph{Sphere.}
For \(x,y\in\sS^{d-1}\), we use the von Mises-Fisher kernel
\begin{equation}
k_{\sS}(x,y)
=
\exp\!\left[\kappa\left(x^\top y-1\right)\right],
\label{equation:mmd_sphere_kernel}
\end{equation}
where \(\kappa>0\) controls the concentration of the kernel. Since \(x\) and \(y\) are unit vectors, \(x^\top y\) depends only on their relative position on the sphere, and the subtraction of \(1\) normalizes the kernel such that \(k_{\sS}(x,x)=1\).

\paragraph{Torus.}
For \(x,y\in\sT^d\), represented by angular coordinates \(x=(x^1,\ldots,x^d)\) and \(y=(y^1,\ldots,y^d)\), we use the product von Mises kernel
\begin{equation}
k_{\sT}(x,y)
=
\exp\!\left[
\kappa
\sum_{j=1}^{d}
\left(
\cos(x^j-y^j)-1
\right)
\right].
\label{equation:mmd_torus_kernel}
\end{equation}
Equivalently, this kernel is the product of one-dimensional circular von Mises kernels over the \(d\) coordinates. It is invariant to shifts of any angular coordinate by integer multiples of \(2\pi\) and therefore respects the periodic structure of \(\sT^d\).

\paragraph{Rotation group.}
For \(\mathrm{SO}(3)\), rotations are represented by unit quaternions \(q,q'\in\sS^3\), with \(q\) and \(-q\) corresponding to the same rotation. We first compute their rotation angle
\begin{equation}
\theta(q,q')
=
2\arccos\!\left(\left|q^\top q'\right|\right)
\in[0,\pi].
\label{equation:so3_mmd_distance}
\end{equation}
The absolute inner product makes the distance invariant to the quaternion sign ambiguity. We then use a normalized heat kernel on \(\mathrm{SO}(3)\),
\begin{equation}
k_{\mathrm{SO}(3)}(q,q')
=
\frac{
\displaystyle
\sum_{\ell=0}^{\infty}
(2\ell+1)
e^{-\tau\ell(\ell+1)}
\chi_\ell\!\left(\theta(q,q')\right)
}{
\displaystyle
\sum_{\ell=0}^{\infty}
(2\ell+1)^2
e^{-\tau\ell(\ell+1)}
},
\label{equation:mmd_so3_kernel}
\end{equation}
where \(\tau>0\) is the heat time and
\begin{equation}
\chi_\ell(\theta)
=
1+2\sum_{m=1}^{\ell}\cos(m\theta)
=
\frac{\sin\left((\ell+\frac{1}{2})\theta\right)}
{\sin(\theta/2)}
\end{equation}
is the character of the \((2\ell+1)\)-dimensional irreducible representation of \(\mathrm{SO}(3)\). The denominator corresponds to the value of the unnormalized kernel at the identity and therefore ensures \(k_{\mathrm{SO}(3)}(q,q)=1\). In practice, the spectral expansion is truncated once the remaining tail is below a numerical tolerance.

\paragraph{Special Euclidean group.}
For \(x=(q,t)\) and \(y=(q',t')\) in \(\mathrm{SE}(3)\), where \(q,q'\in\mathrm{SO}(3)\) denote rotations and \(t,t'\in\mathbb R^3\) translations, we use a product kernel combining the rotational heat kernel above with a Gaussian kernel on the translation component:
\begin{equation}
k_{\mathrm{SE}(3)}(x,y)
=
k_{\mathrm{SO}(3)}^{\,\tau/\alpha^2}(q,q')
\exp\!\left(
-\frac{\|t-t'\|_2^2}{4\tau}
\right).
\label{equation:mmd_se3_kernel}
\end{equation}
Here, \(k_{\mathrm{SO}(3)}^{\,\tau/\alpha^2}\) denotes the \(\mathrm{SO}(3)\) heat kernel in \eqref{equation:mmd_so3_kernel} evaluated with heat time \(\tau/\alpha^2\), while \(\alpha>0\) controls the relative scale between rotational and translational differences. This product construction respects the \(\mathbb R^3\times\mathrm{SO}(3)\) representation used for \(\mathrm{SE}(3)\).

\paragraph{Kernel parameters.}
For the sphere and torus experiments, the kernel concentration is selected separately for each dimension using the scale of the target distribution. Specifically, we compute the median positive pairwise manifold distance \(s\) among the first \(200\) target samples and set \(\kappa=1/s^2\). For the sphere, the distance used for this bandwidth selection is the spherical angle \(d_{\sS}(x,y)=\arccos(x^\top y)\), while for the torus we use the wrapped distance \(d_{\sT}(x,y)=\left(\sum_{j=1}^{D}\operatorname{wrap}(x^j-y^j)^2\right)^{1/2}\), where \(\operatorname{wrap}(\delta)=\operatorname{atan2}(\sin\delta,\cos\delta)\).

For \(D\in\{2,10,20,40,60,80,100\}\), the resulting sphere concentrations are \(\kappa_{\sS}\in\{4.071,0.700,0.438,0.458,0.808,2.145,10.157\}\), respectively. For the same dimensions, the torus concentrations are \(\kappa_{\sT}\in\{4.010,0.5947,0.2873,0.1412,0.09363,0.07003,0.05593\}\), respectively.

For \(\mathrm{SO}(3)\), we use the heat kernel with a fixed heat time \(\tau=0.5\) for all experiments. For \(\mathrm{SE}(3)\), we likewise use \(\tau=0.5\), with the rotational factor evaluated at heat time \(\tau/\alpha^2\) and \(\alpha=2\), as in the kernel definition above.

\paragraph{Crystallographic texture metrics.}
For the crystallographic texture experiment, let
\(\{y^{(i)}\}_{i=1}^{N}\) and
\(\{\widehat y^{(j)}\}_{j=1}^{M}\)
denote the target and transported orientations, respectively. We estimate their orientation distribution functions (ODFs) using the same \(\mathrm{SO}(3)\) heat kernel introduced above. Specifically, we use the heat-kernel density
\begin{equation}
\widetilde{k}_{\mathrm{SO}(3)}^{\,\tau}(q,q')
=
\sum_{\ell=0}^{\infty}
(2\ell+1)
e^{-\tau\ell(\ell+1)}
\chi_\ell\!\left(\theta(q,q')\right),
\end{equation}
which corresponds to the numerator of \eqref{equation:mmd_so3_kernel}. Unlike the MMD kernel, no normalization by its value at the identity is required here. The target and predicted ODFs are estimated as
\begin{equation}
f_{\mathrm{tgt}}(g)
=
\frac{1}{N}
\sum_{i=1}^{N}
\widetilde{k}_{\mathrm{SO}(3)}^{\,\tau_{\mathrm{tex}}}
\left(g,y^{(i)}\right),
\qquad
f_{\mathrm{pred}}(g)
=
\frac{1}{M}
\sum_{j=1}^{M}
\widetilde{k}_{\mathrm{SO}(3)}^{\,\tau_{\mathrm{tex}}}
\left(g,\widehat y^{(j)}\right),
\end{equation}
where \(g\in\mathrm{SO}(3)\) and we set \(\tau_{\mathrm{tex}}=0.01\).

The \emph{texture index} measures the concentration of an ODF and is defined as
\begin{equation}
\mathrm{TI}(f)
=
\int_{\mathrm{SO}(3)} f(g)^2\,dg.
\end{equation}
We report the relative texture-index error
\begin{equation}
E_{\mathrm{TI}}
=
\frac{
\left|
\mathrm{TI}(f_{\mathrm{pred}})
-
\mathrm{TI}(f_{\mathrm{tgt}})
\right|
}{
\mathrm{TI}(f_{\mathrm{tgt}})
}.
\end{equation}
In practice, these integrals are evaluated directly from the orientation samples using the heat-kernel identity, resulting in pairwise evaluations of
\(\widetilde{k}_{\mathrm{SO}(3)}^{\,2\tau_{\mathrm{tex}}}\); diagonal self-pairs are excluded when estimating each texture index.

The \emph{normalized ODF difference} measures the discrepancy between the complete predicted and target orientation distributions:
\begin{equation}
E_{\mathrm{ODF}}
=
\frac{
\int_{\mathrm{SO}(3)}
\left(
f_{\mathrm{pred}}(g)-f_{\mathrm{tgt}}(g)
\right)^2\,dg
}{
\int_{\mathrm{SO}(3)}
f_{\mathrm{tgt}}(g)^2\,dg
}.
\end{equation}

Finally, the \emph{M-index} characterizes texture through the distribution of crystallographic misorientation angles. We compute HCP symmetry-aware disorientation angles from pairs of orientations and construct normalized histograms
\(p_{\mathrm{tgt}}\) and \(p_{\mathrm{pred}}\). Given the corresponding histogram \(p_{\mathrm{rand}}\) for uniformly random orientations, the M-index is
\begin{equation}
M(p)
=
\frac{1}{2}
\sum_{b=1}^{B}
\left|
p^{(b)}-p_{\mathrm{rand}}^{(b)}
\right|,
\end{equation}
and we report
\begin{equation}
E_{\mathrm{M}}
=
\left|
M(p_{\mathrm{pred}})
-
M(p_{\mathrm{tgt}})
\right|.
\end{equation}
We use \(B=180\) bins, up to \(4096\) orientation pairs for each distribution, and \(100{,}000\) random orientation pairs to construct \(p_{\mathrm{rand}}\). Lower values indicate better agreement with the target for all three metrics.

Quantitative comparison of crystallographic textures is commonly based on descriptors of the orientation distribution and crystallographic misorientations. The texture index, also known as the J-index, is widely used to quantify the strength or concentration of an orientation distribution, while the M-index was introduced by \citet{SKEMER2005157} as a complementary measure based on the distribution of crystallographic misorientation angles. Both measures have subsequently been used in EBSD-based texture analysis~\citep{ottseismic2024}. To directly compare two orientation distributions, normalized squared differences between their ODFs have also been used in deformation-texture and crystal-plasticity studies to quantify agreement between predicted and measured textures~\citep{VANHOUTTE2005589}. Following these approaches, we use the texture index error, normalized ODF difference, and M-index error to evaluate complementary aspects of the agreement between the transported and target crystallographic textures.

\subsection{Implementation Details}
\label{appendix:implementation_details}

For each experiment, all methods use the same training and test distributions and are evaluated on the same test samples. Unless stated otherwise, we use minibatches of size \(256\). ManifoldLightOT, ERNOT, and NeuralOT are implemented in PyTorch, while RNOT follows its original JAX/Flax implementation. All experiments are run on a single GPU. The hyperparameters described below are those used for the reported experiments.

\textbf{ManifoldLightOT.}
ManifoldLightOT directly optimizes the parameters of the geometry-specific mixture potentials introduced in \S\ref{subsection:tractable_parameterizations}, without using an additional neural network. The mixture components are initialized from randomly selected target training samples, with initially uniform mixture weights, and all parameters are optimized with Adam using \(\beta_1=0.9\), \(\beta_2=0.999\), and no weight decay.

\emph{Synthetic experiments.}
For the high-dimensional sphere and torus experiments, we use \(K=30\) components in dimension \(2\) and \(K=10D\) components for \(D\in\{10,20,40,60,80,100\}\), with batch size \(256\). We train for \(50000\) steps for \(D\leq20\), \(100000\) steps for \(D\in\{40,60,80\}\), and \(200000\) steps for \(D=100\). The learning rate is \(10^{-3}\) for the low- and medium-dimensional configurations and \(10^{-2}\) for the higher-dimensional configurations. We use \(\varepsilon=0.1\) up to dimension \(40\) and \(\varepsilon=0.01\) for dimensions \(60\)--\(100\), with \(\sT^2\) also using \(\varepsilon=0.01\). For the synthetic \(\mathrm{SO}(3)\) and \(\mathrm{SE}(3)\) experiments, we use \(K=100\) components, batch size \(256\), and train for \(50{,}000\) steps with learning rate \(10^{-2}\), using \(\varepsilon=0.05\) and \(0.2\), respectively. The \(\mathrm{SO}(3)\) rotational kernel scale is \(\lambda_R=1\), while for \(\mathrm{SE}(3)\) we use rotational scale \(1\) and translational scale \(0.16\).

\emph{Continental Drift.}
We use \(K=1000\) components, \(\varepsilon=0.1\), batch size \(64\), and train for \(10^6\) steps with learning rate \(10^{-2}\).

\emph{Crystallographic texture.}
We use \(K=1000\) components, \(\varepsilon=0.01\), rotational kernel scale \(\lambda_R=2\), batch size \(256\), and train for \(200000\) steps with learning rate \(10^{-2}\).

\textbf{RNOT.}
Our RNOT implementation is based on the official implementation provided by the authors (\url{https://openreview.net/forum?id=ez4oLq7PR3}). The potential is parameterized using the Gromov distance embedding of the input with respect to a set of manifold landmarks, followed by an MLP with two hidden layers of \(128\) units, SiLU activations, and LayerNorm. The inner optimization required to evaluate the \(c\)-transform is performed with Adam using \(\beta_1=0.9\), \(\beta_2=0.999\), and tolerance \(10^{-6}\).

\emph{Synthetic experiments.}
We use \(128\) randomly selected landmarks, batch size \(256\), learning rate \(10^{-3}\), and no learning-rate decay. The inner optimization uses \(2500\) iterations with learning rate \(5\times10^{-2}\). Most experiments are trained for \(1000\) outer steps, increased to \(1500\) for \(\sT^{40}\) and \(2000\) for \(\sT^{80}\), \(\sT^{100}\), and \(\mathrm{SE}(3)\).

\emph{Continental Drift.}
We use \(1024\) landmarks obtained by farthest-point sampling from \(4096\) candidates, batch size \(1024\), and train for \(1000\) outer steps with learning rate \(10^{-3}\) and no learning-rate decay. The inner optimization uses \(1000\) iterations with learning rate \(5\times10^{-2}\).

\emph{Crystallographic texture.}
We use \(128\) randomly selected landmarks, batch size \(256\), and train for \(10{,}000\) outer steps with learning rate \(10^{-3}\) and no learning-rate decay. The inner optimization uses \(2500\) iterations with learning rate \(5\times10^{-2}\).

\textbf{ERNOT.}
We parameterize the ERNOT Schr\"odinger potential using distances to manifold landmarks, followed by a two-layer MLP with \(256\) hidden units per layer, SiLU activations, and LayerNorm on the landmark embedding. The output layer is initialized with standard deviation \(0.01\). The regularization parameter \(\varepsilon\) is estimated separately for each dataset from \(256\) training samples using the scale factor \(0.05\). The potential is optimized with Adam with initial learning rate \(10^{-3}\), together with cosine annealing to zero.

\emph{Synthetic experiments.}
We use \(256\) landmarks, batch size \(256\), and train for \(3000\) steps with initial learning rate \(10^{-3}\).

\emph{Continental Drift.}
We use \(256\) landmarks, batch size \(64\), and train for \(3000\) steps with initial learning rate \(10^{-3}\).

\emph{Crystallographic texture.}
We use \(256\) landmarks, batch size \(256\), and train for \(100{,}000\) steps with initial learning rate \(10^{-3}\). At inference time, we use a target support of \(200\) samples.

\textbf{NeuralOT.}
Our NeuralOT implementation is based on the official implementation (\url{https://github.com/iamalexkorotin/NeuralOptimalTransport}), adapted to the manifold representations and transport costs considered in this work. Both the transport map \(T\) and the potential \(f\) are represented by MLPs with three hidden layers and ReLU activations. The transport network maps between the corresponding manifold coordinate representations and its output is projected back to the manifold, while the potential network outputs a scalar. We use hidden widths of \(256\) or \(512\), depending on the experimental configuration. Training follows the alternating minimax procedure of NeuralOT: several transport-map updates are performed with the potential frozen, followed by one potential update with the transport map frozen. Both networks are optimized independently with Adam using \(\beta_1=0.9\), \(\beta_2=0.999\), and no weight decay. NeuralOT uses the manifold-specific costs corresponding to the kernels in Table~\ref{table:manifold_parameterizations}.

\emph{Synthetic experiments.}
Both networks are optimized with learning rate \(10^{-4}\). We use \(10\) transport-map updates per potential update for all configurations except \(\sT^2\), for which we use \(5\). The batch size is \(256\), except for \(\sT^{40}\), where we use \(512\). On the sphere, we train for \(10{,}000\) steps for \(D\in\{2,10\}\), \(20{,}000\) steps for \(D\in\{20,40\}\), \(30{,}000\) steps for \(D\in\{60,80\}\), and \(50{,}000\) steps for \(D=100\). On the torus, we use \(10{,}000\) steps for \(D\in\{2,10,20\}\) and \(20{,}000\) steps for \(D\in\{40,60,80,100\}\). The \(\mathrm{SO}(3)\) and \(\mathrm{SE}(3)\) experiments are trained for \(20{,}000\) steps.

\emph{Continental Drift.}
NeuralOT is not included in the Continental Drift experiments.

\emph{Crystallographic texture.}
We use hidden width \(512\), batch size \(256\), and \(5\) transport-map updates per potential update. Both networks are optimized with learning rate \(10^{-4}\), and training is performed for \(20{,}000\) steps.

\subsection{Additional Results}
\label{appendix:additional_results}

In this section we present additional results for our synthetic and crystallization experimental setups.

\textbf{Additional High Dimensional Spheres and Tori.} In Table \ref{table:high_dim_sphere_torus} we show additional dimensionalities which complements the ones appearing in Table \ref{table:synthetic_results}. It is important to note that ManifoldLightOT remains showing a strong performance against other methods.

\begin{table*}[h]
\centering
\caption{MMD$^2$ results on high-dimensional spheres and tori, reported in units of $10^{-2}$. We report mean $\pm$ standard deviation over five seeds. Lower is better.}
\label{table:high_dim_sphere_torus}
\resizebox{\textwidth}{!}{
\begin{tabular}{lccccccc}
\toprule
\multicolumn{8}{c}{\textbf{Sphere}} \\
\midrule
Method & $\sS^{2}$ & $\sS^{10}$ & $\sS^{20}$ & $\sS^{40}$ & $\sS^{60}$ & $\sS^{80}$ & $\sS^{100}$ \\
\midrule
ManifoldLightOT & $0.08 \pm 0.01$ & $0.07 \pm 0.01$ & $0.08 \pm 0.01$ & $0.06 \pm 0.00$ & $0.07 \pm 0.00$ & $0.13 \pm 0.01$ & $0.45 \pm 0.05$ \\
RNOT            & $0.13 \pm 0.06$ & $0.10 \pm 0.02$ & $0.10 \pm 0.01$ & $0.11 \pm 0.02$ & $0.18 \pm 0.04$ & $0.43 \pm 0.09$ & $1.47 \pm 0.21$ \\
ERNOT           & $0.12 \pm 0.03$ & $0.15 \pm 0.01$ & $0.20 \pm 0.03$ & $0.20 \pm 0.04$ & $0.22 \pm 0.01$ & $0.25 \pm 0.02$ & $0.32 \pm 0.09$ \\
NeuralOT        & $0.19 \pm 0.07$ & $0.12 \pm 0.04$ & $0.17 \pm 0.02$ & $0.17 \pm 0.02$ & $0.27 \pm 0.06$ & $0.29 \pm 0.04$ & $4.38 \pm 8.24$ \\
\midrule
\multicolumn{8}{c}{\textbf{Torus}} \\
\midrule
Method & $\sT^{2}$ & $\sT^{10}$ & $\sT^{20}$ & $\sT^{40}$ & $\sT^{60}$ & $\sT^{80}$ & $\sT^{100}$ \\
\midrule
ManifoldLightOT & $0.27 \pm 0.06$ & $0.08 \pm 0.02$ & $0.08 \pm 0.01$ & $0.07 \pm 0.00$ & $0.08 \pm 0.00$ & $0.08 \pm 0.01$ & $0.08 \pm 0.00$ \\
RNOT            & $0.18 \pm 0.06$ & $0.20 \pm 0.08$ & $0.20 \pm 0.02$ & $0.39 \pm 0.04$ & $2.15 \pm 0.20$ & $7.49 \pm 0.38$ & $23.15 \pm 0.57$ \\
ERNOT           & $0.14 \pm 0.02$ & $0.22 \pm 0.06$ & $0.27 \pm 0.08$ & $0.27 \pm 0.04$ & $0.27 \pm 0.02$ & $0.26 \pm 0.04$ & $0.27 \pm 0.03$ \\
NeuralOT        & $33.03 \pm 29.80$ & $48.60 \pm 27.08$ & $61.31 \pm 0.20$ & $61.53 \pm 0.09$ & $61.59 \pm 0.06$ & $61.63 \pm 0.05$ & $61.61 \pm 0.07$ \\
\bottomrule
\end{tabular}
}
\end{table*}

\end{document}